\documentclass[11pt]{article}

\usepackage[margin=1in]{geometry}
\usepackage[T1]{fontenc}
\usepackage{microtype}
\usepackage[round,authoryear]{natbib}

\usepackage{amsmath,amssymb,amsfonts,amsthm,bm,mathtools}
\usepackage{booktabs,longtable,array,multirow}
\usepackage{algorithm}
\usepackage{float}
\usepackage{algorithmic}
\usepackage{enumitem}
\usepackage{xcolor}
\usepackage{graphicx}
\usepackage{hyperref}

\hypersetup{
  colorlinks=true,
  linkcolor=blue,
  citecolor=blue,
  urlcolor=blue
}

\newtheorem{theorem}{Theorem}
\newtheorem{lemma}[theorem]{Lemma}
\newtheorem{proposition}[theorem]{Proposition}

\newcommand{\E}{\mathbb{E}}

\newcommand{\calA}{\mathcal{A}}

\newcommand{\calF}{\mathcal{F}}

\newcommand{\calZ}{\mathcal{Z}}

\newcommand{\calX}{\mathcal{X}}
\newcommand{\one}{\mathbf{1}}

\newcommand{\clip}{\operatorname{clip}}
\newcommand{\Var}{\operatorname{Var}}

\newcommand{\KL}{\operatorname{KL}}
\newcommand{\Bern}{\operatorname{Bern}}
\newcommand{\klbin}{\operatorname{kl}}

\newcommand{\argmax}{\mathop{\operatorname{argmax}}}
\newcommand{\argmin}{\mathop{\operatorname{argmin}}}
\newcommand{\ASB}{\textnormal{\textsc{ASB}}}

\hypersetup{
  pdftitle={The Surrogate Is Not the Reward: Post-Surrogate Primary-Outcome Acquisition in Contextual Bandits},
  pdfauthor={Kyungbok Lee and Michael R. Kosorok}
}

\title{The Surrogate Is Not the Reward:\\Post-Surrogate Primary-Outcome Acquisition in Contextual Bandits}
\author{
Kyungbok Lee \qquad Michael R. Kosorok\\
Department of Biostatistics\\
University of North Carolina at Chapel Hill\\
Chapel Hill, NC 27599-7420, USA\\
\texttt{kyungbok@unc.edu} \qquad \texttt{kosorok@unc.edu}
}

\date{}

\begin{document}
\maketitle

\begin{abstract}
We study contextual bandits in which a surrogate is observed after the action but before the learner decides whether to acquire the primary outcome that defines action value and regret. The value of acquiring the primary outcome depends on both decision relevance---how much the current outcome matters for comparing policies---and the residual uncertainty after observing the surrogate. The Audited Surrogate Bandit (\ASB) learns a contextual policy while allocating a budget of $B$ primary-outcome acquisitions over $T$ rounds. \ASB\ sets a pre-surrogate acquisition level from current decision relevance and, after observing the surrogate, redistributes that level using an estimate of that residual uncertainty. For a finite class of $N$ policies over $K$ actions, \ASB\ incurs $\widetilde O[\sqrt{KT\log N}\{1+\sqrt{T/B}\}]$ regret relative to the best policy in the class. In a two-action family where the surrogate does not reveal the better action, a learner that observes the surrogate before deciding whether to acquire can achieve bounded regret, whereas any learner that must decide before seeing the surrogate incurs $\Omega(T/B)$ worst-case regret under the same budget. Synthetic experiments show that both acquisition factors matter: \ASB\ has lower regret than variants using only decision relevance or only residual uncertainty. On a KuaiRec benchmark of user--video interactions, the regret gap relative to relevance-only acquisition widens and then narrows as the budget grows.
\end{abstract}

\section{Introduction}

Post-action information can arrive too late to change the current action yet early enough to affect whether the primary outcome is acquired. We study contextual bandits in which a costly primary outcome $Y_t$ defines action value and regret, but after selecting action $A_t$ in context $X_t$, the learner first observes a surrogate $S_t$ and only then decides whether to acquire $Y_t$. In mobile interventions, proximal activity can be observed automatically after intervention while broader health outcomes are harder to assess \citep{nahumshani2018jitai,klasnja2019heartsteps}. The round follows
\[
X_t\longrightarrow A_t\longrightarrow S_t\longrightarrow \text{acquire }Y_t\text{?}.
\]

By the time $S_t$ arrives, the current action cannot be changed; the remaining decision is whether to measure $Y_t$. This leaves a distinct use for post-action information: it can help decide where to spend a limited measurement budget, while action value and regret remain tied to $Y_t$. The usefulness of measuring $Y_t$ depends on both the current policy comparison and what remains uncertain after $S_t$. If competing policies choose the same action, observing $Y_t$ does not distinguish them on that round; if they disagree but little residual variation remains after $S_t$, observing $Y_t$ provides little information for distinguishing them. This points to two ingredients: \emph{decision relevance}---how much the current outcome matters to unresolved policy comparisons---and the residual uncertainty estimated after the surrogate.

The Audited Surrogate Bandit (\ASB) is an end-to-end contextual-bandit method that uses these two ingredients while learning a policy and deciding which primary outcomes to acquire. Before observing the surrogate, \ASB\ sets an acquisition level from decision relevance; after observing it, \ASB\ refines that level using estimated residual uncertainty. This timing also raises a broader question: what can be gained by allowing the acquisition decision to use the current surrogate? Our contributions are as follows:
\begin{itemize}[leftmargin=1.4em]
\item We provide \ASB, an end-to-end contextual-bandit method that learns a policy while allocating a hard budget of $B$ primary-outcome acquisitions. In \ASB, decision relevance sets the acquisition level before the surrogate is observed, and estimated residual uncertainty refines that level afterward. For a finite class of $N$ policies over $K$ actions, \ASB\ acquires exactly $B$ primary outcomes on every path and achieves $\widetilde O[\sqrt{KT\log N}\{1+\sqrt{T/B}\}]$ regret relative to the best policy in the class. A lower bound shows that the $\sqrt{T/B}$ dependence on the acquisition budget is unavoidable up to logarithmic factors on a subclass where an optimal policy is represented.
\item We compare primary-outcome acquisition before and after observing the surrogate. For a fixed policy comparison under the same mean acquisition rate, allowing acquisition to use $S_t$ can strictly lower the minimum variance of an unbiased estimate of the primary-outcome value difference. We show on a two-action family that this timing can change the attainable regret rate: post-surrogate acquisition can achieve bounded regret while any learner deciding before the surrogate incurs $\Omega(T/B)$ worst-case regret, even though $S_t$ does not reveal the optimal action.
\item We demonstrate in synthetic experiments that both parts of \ASB's acquisition design matter for regret: removing either decision relevance or the post-surrogate residual-RMS refinement increases cumulative regret. These experiments also show that using the current surrogate for acquisition can reduce regret even when it does not reveal the optimal action. \ASB\ has lower regret when more residual information lies where represented policies disagree rather than where they agree, even when the overall amount of residual information is fixed. We evaluate \ASB\ against relevance-only acquisition on KuaiRec recommendation data across several acquisition budgets: the mean-regret gap first widens and then narrows as the budget grows.
\end{itemize}

Related work on acquisition design includes prediction-guided experimentation and selective ground-truth acquisition for inference or best-arm identification \citep{ao2024pgae,ao2026llmaudits}. Costly-feedback bandits study whether to observe reward \citep{efroni2021confidence,tucker2023costly}, while intermediate- and post-action-feedback methods use information revealed after acting \citep{esposito2023intermediate,wang2023followups}. In \ASB, the surrogate is observed after the action and used to decide whether to acquire the same round's primary outcome, while regret and subsequent policy learning remain tied to that primary outcome. Appendix~\ref{app:related_scope} discusses broader connections to surrogate-based inference and limited-follow-up designs, active inference and feature acquisition, and surrogate-guided sequential decision methods.

\section{Post-Surrogate Primary-Outcome Acquisition}
\label{sec:problem}

For $t=1,\ldots,T$, the learner observes a context $X_t\in\calX$ and selects an action $A_t\in\calA$, where $K=|\calA|<\infty$. The resulting surrogate $S_t$ is then observed for free. The pre-surrogate and post-surrogate states are $U_t=(X_t,A_t)$ and $Z_t=(U_t,S_t)$, respectively. After observing $S_t$, the learner sets an acquisition probability $\rho_t(Z_t)\in[0,1]$, draws $Q_t\sim\Bern\{\rho_t(Z_t)\}$, and observes the primary outcome $Y_t\in[0,1]$ iff $Q_t=1$. The acquisition decision only determines whether $Y_t$ is observed; it does not change future contexts or outcomes. With $\calF_{t-1}$ the history before round $t$, the timing is
\[
  (\calF_{t-1},X_t)\to A_t\to S_t\to\rho_t(Z_t)\to Q_t\to Q_tY_t,
  \qquad Q_t\perp Y_t\mid(\calF_{t-1},Z_t).
\]

The conditional mean of the primary outcome is denoted by $r^\star(x,a)$, so $\E(Y_t\mid\calF_{t-1},X_t,A_t)=r^\star(X_t,A_t)$. Cumulative regret is
$R_T:=\sum_{t=1}^T[\max_{a\in\calA}r^\star(X_t,a)-r^\star(X_t,A_t)]$. The learner minimizes $R_T$ subject to the hard budget $\sum_{t=1}^T Q_t\le B$ on every path, where $B\in\{1,\ldots,T\}$ is the primary-outcome acquisition budget.

This timing raises two questions: how to learn a contextual policy while allocating the primary-outcome budget, and what changes when the learner decides whether to acquire the primary outcome before or after observing the current surrogate.

\section{Audited Surrogate Bandit}

We present \ASB\ for the setting in Section~\ref{sec:problem}. The method learns a contextual policy while deciding which primary outcomes to acquire, using a finite class $\Pi$ of $N$ deterministic policies. We first follow one online round from action selection through primary-outcome acquisition and policy updating, then examine why \ASB\ uses decision relevance before the surrogate and residual uncertainty after it.

\subsection{One round of \ASB}

\paragraph{Action selection and decision relevance.}
At round $t$, \ASB\ begins with its current weights $w_t(\pi)$ for policies $\pi\in\Pi$. After observing $X_t$, it normalizes these weights to form a policy mixture $\mu_t(\pi)\propto w_t(\pi)$. This mixture induces the action distribution $\bar p_t(a)=\sum_{\pi\in\Pi}\mu_t(\pi)\one\{\pi(X_t)=a\}$ over $a\in\mathcal A$. Using exploration rate $\varepsilon_t$, \ASB\ forms $p_t(a)=(1-\varepsilon_t)\bar p_t(a)+\varepsilon_tu_t(a)$, where $u_t$ is uniform on the support of $\bar p_t$, and draws $A_t\sim p_t$.

For any $a\in\mathcal A$ with $p_t(a)>0$, define the \emph{policy-comparison coefficient} $c_{t,\pi}(a)$ and $\ell_t(a)$ by
\begin{equation}
\label{eq:centered_policy_influence}
  c_{t,\pi}(a)
  =\{\one\{\pi(X_t)=a\}-\bar p_t(a)\}/p_t(a),
  \qquad
  \ell_t(a)=\max_{\pi'\in\Pi}|c_{t,\pi'}(a)|.
\end{equation}
The coefficient $c_{t,\pi}(a)$ measures how much the primary-outcome mean at action $a$ contributes to the comparison between policy $\pi$ and the current policy mixture $\mu_t$, accounting for the probability $p_t(a)$ of selecting $a$. Its maximum magnitude across policies, $\ell_t(a)$, is the \emph{decision-relevance score}.

\paragraph{Pre-surrogate acquisition level and post-surrogate refinement.}
Before observing $S_t$, \ASB\ uses $\ell_t(A_t)$ to set a pre-surrogate acquisition level $\rho_t^0$. After observing $S_t$, it evaluates the fitted bridge $h_t(Z_t)\in[0,1]$ for $Y_t$ and a learned root-mean-square (RMS) residual score $\widehat\tau_t(Z_t)$ for $Y_t-h_t(Z_t)$. \ASB\ combines decision relevance with residual RMS through the following basic acquisition form:
\begin{equation}
\label{eq:asb_acquisition_geometry_main}
  \rho_t^0\propto\ell_t(A_t),\qquad
  \rho_t^1(Z_t)=\rho_t^0\,\frac{\widehat\tau_t(Z_t)}{\widehat\E_t\{\widehat\tau_t(U_t,S)\mid U_t\}},\qquad
  \rho_t(Z_t)=(1-\gamma_t)\rho_t^0+\gamma_t\rho_t^1(Z_t).
\end{equation}
Here $\widehat\E_t(\cdot\mid U_t)$ is a fitted conditional expectation over $S$. The refinement $\rho_t^1$ redistributes the pre-surrogate level across surrogate values according to the learned residual RMS score. The resulting acquisition probability $\rho_t$ combines the pre-surrogate level with this refinement, with $\gamma_t\in[0,1]$, fixed before $S_t$, controlling how much of the redistribution is used. Appendix~\ref{app:practical_fitted_components} gives one concrete online fitting specification for $h_t$, $\widehat\tau_t$, and $\widehat\E_t$.

Algorithm~\ref{alg:hard_budget_pacing} in Appendix~\ref{app:hard_budget_acquisition} turns the acquisition structure in Equation~\eqref{eq:asb_acquisition_geometry_main} into the hard-budget rule. It uses the remaining budget and past decision-relevance scores to set the scale of $\rho_t^0$. When acquisition is not fixed by the budget rule, the post-surrogate refinement is constrained to keep its fitted conditional mean at $\rho_t^0$ and, whenever $\ell_t(A_t)>0$, to stay above a positive acquisition floor set before $S_t$ without using the learned residual score. Its remaining-count boundary rules ensure
$\sum_{t=1}^T Q_t=B$ on every path.

\paragraph{Augmented feedback and policy update.}
\ASB\ then draws $Q_t\sim\Bern\{\rho_t(Z_t)\}$ and observes $Y_t$ iff $Q_t=1$. For $\rho_t(Z_t)>0$, it forms the augmented feedback
\begin{equation}
\label{eq:audited_reward}
  \Phi_t
  =h_t(Z_t)+\frac{Q_t}{\rho_t(Z_t)}\{Y_t-h_t(Z_t)\}.
\end{equation}
Using $Q_t\perp Y_t\mid(\calF_{t-1},Z_t)$, the correction in \eqref{eq:audited_reward} preserves the primary-outcome conditional mean:
\begin{equation}
\label{eq:target_validity}
  \E(\Phi_t\mid\calF_{t-1},Z_t)
  =\E(Y_t\mid\calF_{t-1},Z_t),
  \qquad
  \E(\Phi_t\mid\calF_{t-1},U_t)=r^\star(X_t,A_t).
\end{equation}
For each policy $\pi$, \ASB\ forms the \emph{policy-comparison evidence} $G_{t,\pi}:=c_{t,\pi}(A_t)\{\Phi_t-1/2\}$, with $1/2$ centering the $[0,1]$ outcome scale. If $\rho_t(Z_t)>0$ whenever $\ell_t(A_t)>0$, Equations~\eqref{eq:centered_policy_influence} and \eqref{eq:target_validity} imply
\begin{equation}
\label{eq:policy_comparison_target_main}
  \E\!\left(G_{t,\pi}\mid\calF_{t-1},X_t\right)
  =
  r^\star\{X_t,\pi(X_t)\}
  -\sum_a\bar p_t(a)r^\star(X_t,a).
\end{equation}
Then $G_{t,\pi}$ is conditionally unbiased for the primary-outcome mean difference between policy $\pi$ and the policy mixture $\mu_t$. \ASB\ then updates the policy weights using the second-order exponential-weights rule with learning rate $\eta_t$:
\begin{equation}
\label{eq:observed_second_order_update}
  w_{t+1}(\pi)
  =w_t(\pi)\exp\{\eta_tG_{t,\pi}+\eta_t^2G_{t,\pi}^2\}.
\end{equation}
\ASB\ also updates $h_{t+1}$, $\widehat\tau_{t+1}$, and $\widehat\E_{t+1}$, and the next round starts from the new policy weights $w_{t+1}$.

\begin{algorithm}[H]
\caption{Audited Surrogate Bandit (\ASB)}
\label{alg:asb}
\algsetup{indent=.8em}
\small
\begin{algorithmic}[1]
\REQUIRE horizon $T$, budget $B$, policy class $\Pi$; exploration/refinement schedules $(\varepsilon_t,\gamma_t)_{t=1}^T$; policy learning-rate rule for $\eta_t$; fitting procedures for $h_t$, $\widehat\tau_t$, and $\widehat\E_t$.
\STATE \textbf{Initialize:} set $w_1(\pi)=1$ for every $\pi\in\Pi$.
\FOR{$t=1,\ldots,T$}
  \STATE \textbf{Policy mixture:} observe $X_t$; set $\mu_t(\pi)=w_t(\pi)/\sum_{\pi'}w_t(\pi')$; $\bar p_t(a)=\sum_{\pi\in\Pi}\mu_t(\pi)\one\{\pi(X_t)=a\}$.
  \STATE \textbf{Act:} set $p_t(a)=(1-\varepsilon_t)\bar p_t(a)+\varepsilon_tu_t(a)$; draw $A_t\sim p_t$.
  \STATE \textbf{Decision relevance:} compute $c_{t,\pi}(a)$ and $\ell_t(a)$ by \eqref{eq:centered_policy_influence}.
  \STATE \textbf{Acquisition probability:} use Algorithm~\ref{alg:hard_budget_pacing} in Appendix~\ref{app:hard_budget_acquisition} to set $\rho_t^0$ before observing $S_t$ and $\rho_t^1(Z_t)$ after observing $S_t$. Set $\rho_t(Z_t)=(1-\gamma_t)\rho_t^0+\gamma_t\rho_t^1(Z_t)$.
  \STATE \textbf{Acquire primary outcome:} draw $Q_t\sim\Bern\{\rho_t(Z_t)\}$; observe $Y_t$ iff $Q_t=1$.
  \STATE \textbf{Augmented feedback:} if $\rho_t(Z_t)>0$, set $\Phi_t=h_t(Z_t)+\{Q_t/\rho_t(Z_t)\}\{Y_t-h_t(Z_t)\}$ and $G_{t,\pi}=c_{t,\pi}(A_t)\{\Phi_t-1/2\}$; otherwise set $G_{t,\pi}=0$.
  \STATE \textbf{Update:} set $w_{t+1}(\pi)$ by \eqref{eq:observed_second_order_update}; update $h_{t+1}$, $\widehat\tau_{t+1}$, and $\widehat\E_{t+1}$.
\ENDFOR
\end{algorithmic}
\end{algorithm}

\subsection{Acquisition structure: decision relevance and residual RMS}
\label{sec:policy_acquisition}

To explain why Equation~\eqref{eq:asb_acquisition_geometry_main} uses decision relevance before the surrogate and residual RMS afterward, consider how the acquisition probability $\rho_t$ affects the one-round mean-squared error of the policy-comparison evidence $G_{t,\pi}$. Under the positive-support condition in Equation~\eqref{eq:policy_comparison_target_main}, this mean-squared error is its conditional variance, which decomposes as
\begin{equation}
\label{eq:online_comparison_acquisition_term}
  \operatorname{Var}\!\left(G_{t,\pi}\mid\calF_{t-1},X_t\right)
  =
  \E\!\left[
    \left.
    \frac{c_{t,\pi}^2(A_t)\tau_t^2(Z_t)}
         {\rho_t(Z_t)}
    \right|\calF_{t-1},X_t
  \right]
  +\text{a term independent of }\rho_t .
\end{equation}
Here $\tau_t(U_t,s)$ is the \emph{true residual RMS} around the bridge $h_t$:
\[
  \tau_t^2(U_t,s)
  :=\E\!\left[\{Y_t-h_t(U_t,s)\}^2
      \mid\calF_{t-1},U_t,S_t=s\right].
\]
Equation~\eqref{eq:online_comparison_acquisition_term} shows why \ASB\ uses the decision-relevance score $\ell_t(A_t)=\max_{\pi'}|c_{t,\pi'}(A_t)|$ for the pre-surrogate acquisition level and residual RMS $\tau_t(U_t,S_t)$ for post-surrogate redistribution.

\paragraph{Pre-surrogate acquisition level: decision relevance.}
To characterize the action dependence of $\rho_t^0$, consider a candidate pre-surrogate rule $\varrho$, with level $\varrho(a)$ when $A_t=a$. Because $A_t\sim p_t$, Equation~\eqref{eq:online_comparison_acquisition_term} averages these action-specific terms with weights $p_t(a)$. Since $\tau_t^2\le1$ and $\ell_t(a)$ bounds $|c_{t,\pi}(a)|$ over $\pi\in\Pi$, the acquisition-dependent term is at most $\sum_{a:\,\ell_t(a)>0}p_t(a)\ell_t^2(a)/\varrho(a)$ for every $\pi\in\Pi$. At a fixed average acquisition level, this calculation gives the proportional shape $\varrho(a)\propto\ell_t(a)$, motivating $\rho_t^0\propto\ell_t(A_t)$ in Equation~\eqref{eq:asb_acquisition_geometry_main}.

\paragraph{Post-surrogate redistribution: residual RMS.}
The pre-surrogate argument above determines how $\rho_t^0$ depends on the selected action. The post-surrogate question is how to redistribute that level across surrogate values for a fixed $U_t$. The residual RMS $\tau_t(U_t,S)$ may vary with $S$ but is unknown, so \ASB\ uses the learned residual RMS score $\widehat\tau_t(U_t,S)$. Among candidate redistribution rules $r(\cdot)$ with fitted conditional mean $\rho_t^0$, consider
\begin{equation}
\label{eq:conditional_refinement_variational_main}
  \min_{r(\cdot)>0}
  \widehat\E_t\!\left[
    \left.
    \frac{\widehat\tau_t^2(U_t,S)}{r(S)}
    \right|U_t
  \right]
  \quad\text{subject to}\quad
  \widehat\E_t\{r(S)\mid U_t\}=\rho_t^0 .
\end{equation}
When $\widehat\tau_t(U_t,S)>0$ almost surely under the fitted law, its proportional solution gives the basic refinement in Equation~\eqref{eq:asb_acquisition_geometry_main}; Appendix~\ref{app:hard_budget_acquisition} additionally enforces the probability bounds and support floor. Surrogate values with larger learned residual RMS scores receive more acquisition probability. The hard-budget refinement keeps its fitted conditional mean at $\rho_t^0$, so the final acquisition probability preserves that level as well:
\[
  \widehat\E_t\{\rho_t(U_t,S)\mid U_t\}
  =(1-\gamma_t)\rho_t^0
  +\gamma_t\widehat\E_t\{\rho_t^1(U_t,S)\mid U_t\}
  =\rho_t^0.
\]

Under the same support condition, $\rho_t(Z_t)$ may depend on $S_t$, yet Equations~\eqref{eq:target_validity} and~\eqref{eq:policy_comparison_target_main} show that the conditional mean of $G_{t,\pi}$ remains unchanged---the primary-outcome mean difference between policy $\pi$ and the current policy mixture $\mu_t$. Equation~\eqref{eq:online_comparison_acquisition_term}, by contrast, shows that the conditional variance of $G_{t,\pi}$ can change with $\rho_t$. Acquisition rules can therefore preserve that same primary-outcome mean difference while differing in whether $\rho_t$ may depend on $S_t$.

\section{Primary-Outcome Acquisition Before and After the Surrogate: Allocation and Regret}
\label{sec:theory_efficiency}

We now isolate the effect of acquisition timing by comparing rules that choose $\rho_t$ before versus after observing $S_t$. We first ask when allowing $\rho_t$ to depend on $S_t$ can reduce the minimum variance for a fixed policy comparison, and then whether the same timing difference can change the regret rate.

Throughout this section, we use PRE for acquisition rules whose $\rho_t$ is fixed before $S_t$ is observed, and POST for rules whose $\rho_t$ may depend on $S_t$; every PRE rule is therefore also a POST rule.

\subsection{Acquisition Timing for a Policy Comparison}
\label{sec:timing}

Decision relevance describes how much a primary outcome helps distinguish competing policies. \ASB\ quantifies it through comparisons with its current policy mixture; here we fix one policy comparison and compare PRE and POST acquisition rules.

\paragraph{Acquisition objective.}
Fix two deterministic policies $\pi_1,\pi_2$ and a common joint law for $(X,A,S,Y)$. Under this law, $A\sim p(\cdot\mid X)$, where $p$ is the action-selection distribution and $p(a\mid X)>0$ for every action chosen by either policy. Write $U=(X,A)$ for the pre-surrogate state. For this fixed comparison, $|g(U)|$ captures decision relevance through
\[
  g(U)
  =\frac{\one\{\pi_1(X)=A\}-\one\{\pi_2(X)=A\}}{p(A\mid X)}.
\]
Under an acquisition allocation $\rho$ and a fixed bridge $h(U,S)$, the augmented feedback from Equation~\eqref{eq:audited_reward} is $\Phi=h(U,S)+\{Q/\rho(U,S)\}\{Y-h(U,S)\}$ whenever $\rho(U,S)>0$. Combined with $g(U)$, Equation~\eqref{eq:target_validity} gives
\[
\E\!\left[g(U)\{\Phi-1/2\}\right]
=
\E\!\left[
  r^\star\{X,\pi_1(X)\}-r^\star\{X,\pi_2(X)\}
\right].
\]
Thus $g(U)\{\Phi-1/2\}$ is unbiased for the same primary-outcome value difference. Acquisition timing can therefore change its variance without changing its target.

To characterize this variance, fix the bridge $h$ and define the residual RMS $\tau_h(U,S)$ by $\tau_h^2(U,S):=\E[\{Y-h(U,S)\}^2\mid U,S]$. The part of the variance controlled by the allocation $\rho$ is
\begin{equation}
\label{eq:acquisition_variance_burden}
  \mathcal V(\rho):=\E\!\left[\frac{g^2(U)\tau_h^2(U,S)}{\rho(U,S)}\right].
\end{equation}
In $\mathcal V(\rho)$, decision relevance and residual RMS appear as the product $|g(U)|\tau_h(U,S)$. Here $|g(U)|$ is determined by $U$, whereas $\tau_h(U,S)$ may depend on $S$. The contribution vanishes whenever either is zero; finite $\mathcal V(\rho)$ requires $\rho(U,S)>0$ wherever both are positive.

PRE and POST therefore minimize $\mathcal V(\rho)$ under different timing restrictions: PRE restricts $\rho$ to $\rho(U)$, whereas POST allows $\rho(U,S)$. The theorem below gives the two minimizers and identifies when allowing $\rho$ to depend on $S$ strictly lowers the minimum variance.

\begin{theorem}[PRE/POST variance characterization for a fixed policy comparison]
\label{thm:timing_value}
For $\bar\rho\in(0,1]$, consider minimizing $\mathcal V(\rho)$ subject to $\E\rho\le\bar\rho$. If $\Pr\{|g(U)|\tau_h(U,S)>0\}>0$, the PRE and POST minimizers, denoted $\rho_{\rm pre}^\star$ and $\rho_{\rm post}^\star$, are
\begin{equation}
\label{eq:frontier_optimal_allocations}
\begin{aligned}
  \rho_{\rm pre}^\star(U)
  &=
  \bar\rho\,
  \frac{
    |g(U)|\,\{\E[\tau_h^2(U,S)\mid U]\}^{1/2}
  }{
    \E\!\left[
      |g(U)|\,\{\E[\tau_h^2(U,S)\mid U]\}^{1/2}
    \right]
  },\\
  \rho_{\rm post}^\star(U,S)
  &=
  \bar\rho\,
  \frac{
    |g(U)|\,\tau_h(U,S)
  }{
    \E[|g(U)|\,\tau_h(U,S)]
  },
\end{aligned}
\end{equation}
provided $\rho_{\rm pre}^\star(U)\le1$ and $\rho_{\rm post}^\star(U,S)\le1$ almost surely. Moreover, $\mathcal V(\rho_{\rm post}^\star)\le\mathcal V(\rho_{\rm pre}^\star)$, with strict inequality exactly when
$\Pr\!\left(|g(U)|>0,\ \Var\{\tau_h(U,S)\mid U\}>0\right)>0$.
\end{theorem}

The decision-relevance factor $|g(U)|$ is the same in $\rho_{\rm pre}^\star$ and $\rho_{\rm post}^\star$. $\rho_{\rm pre}^\star(U)$ uses the $U$-conditional residual RMS $\{\E[\tau_h^2(U,S)\mid U]\}^{1/2}$, whereas $\rho_{\rm post}^\star(U,S)$ uses $\tau_h(U,S)$ after observing $S$. Hence POST strictly lowers the minimum variance when, at pre-surrogate states $U$ that matter to the policy comparison, observing $S$ adds information about residual uncertainty beyond that in $U$ alone.

\subsection{Regret-Rate Separation}
\label{sec:regret_separation}

For a fixed policy comparison, the preceding result shows that using the current surrogate can strictly reduce the minimum variance at the same average acquisition rate. We next ask whether this timing difference can also separate regret rates when the surrogate does not identify the optimal action.

For any fixed $q\in(0,1)$, consider two bandit instances indexed by $\vartheta\in\{-1,+1\}$. Both have a constant context $X_t\equiv x_0$, binary action set $\mathcal A=\{-1,+1\}$, and the same i.i.d. binary surrogate law $S_t\sim\Bern(q)$, with $S_t\perp A_t\mid\calF_{t-1}$. Define the primary outcome by
\[
  Y_t=1/2\quad\text{if }S_t=0,
  \qquad
  Y_t=\one\{A_t=\vartheta\}\quad\text{if }S_t=1.
\]
In instance $\vartheta$, the primary-outcome mean is $\E_\vartheta(Y_t\mid A_t=a)=1/2+(q/2)\vartheta a$. Thus $a=\vartheta$ is optimal, while choosing $a=-\vartheta$ incurs regret $q$. The surrogate does not reveal the optimal action, but it does reveal whether the current primary outcome is informative: if $S_t=1$, $(A_t,Y_t)$ identifies $\vartheta$; if $S_t=0$, $Y_t=1/2$ under both instances. For the fixed bridge $h\equiv1/2$, the residual RMS is $0$ at $S_t=0$ and $1/2$ at $S_t=1$.

Let $\mathfrak L_{\rm pre}(B)$ denote PRE strategies satisfying
$\E_\vartheta\sum_{t=1}^T Q_t\le B$ for each $\vartheta$. Every PRE strategy under the pathwise hard budget of Section~\ref{sec:problem} belongs to this class, so the lower bound also applies to the original hard-budget problem. Write $(x)_+=\max\{x,0\}$.

\begin{theorem}[Regret separation from acquisition timing]
\label{thm:one_shot_separation_main}
For any $q\in(0,1)$ and $B\in\{1,\ldots,T\}$, every $\mathcal L\in\mathfrak L_{\rm pre}(B)$ satisfies
\begin{equation}
\label{eq:one_shot_pre_lower_main}
  \max_{\vartheta\in\{-1,+1\}}
  \E_\vartheta R_T
  \ge (qT/2)(1-qB)_+.
\end{equation}
In contrast, there exists a POST strategy $\mathcal L_{\rm post}$ such that $\sum_{t=1}^T Q_t=B$ on every path and
\begin{equation}
\label{eq:one_shot_post_upper_main}
  \max_{\vartheta\in\{-1,+1\}}
  \E_\vartheta R_T
  \le [1-(1-q)^T]/2\le1/2.
\end{equation}
\end{theorem}

PRE obtains informative evidence only when an acquisition coincides with $S_t=1$, so the expected number of informative acquisitions is at most $qB$. With $q=1/(2B)$, the probability of no informative acquisition is at least $1/2$. POST can instead acquire the first $S_t=1$ round, which identifies $\vartheta$. Thus
\begin{equation}
\label{eq:one_shot_rate_gap_main}
\begin{aligned}
  \max_\vartheta\E_\vartheta R_T
  &\le1/2
  &&\text{for }\mathcal L_{\rm post},\\
  \max_\vartheta\E_\vartheta R_T
  &\ge T/(8B)
  &&\text{for every }\mathcal L\in\mathfrak L_{\rm pre}(B).
\end{aligned}
\end{equation}
For $B=o(T)$, every $\mathcal L\in\mathfrak L_{\rm pre}(B)$ has $\Omega(T/B)$ worst-case regret, while $\mathcal L_{\rm post}$ has bounded regret. The surrogate reveals when the current primary outcome is informative without revealing the optimal action. Appendix~\ref{app:regret_rate_separation_proof} gives the corresponding separation with noisy informative outcomes.

\section{Regret under a Primary-Outcome Acquisition Budget}

We first establish \ASB's regret guarantee under the pathwise hard budget,
then give a lower bound under a weaker expected acquisition budget to isolate
the budget dependence that remains after this relaxation.

\subsection{ASB Regret under the Hard Budget}
\label{sec:asb_regret_guarantee}

\ASB\ competes with the best policy in $\Pi$, which need not select the best action on every round. To separate this gap from online learning regret, define
$\operatorname{App}_T(\Pi):=\min_{\pi\in\Pi}\sum_{t=1}^T[\max_{a\in\calA} r^\star(X_t,a)-r^\star\{X_t,\pi(X_t)\}]$
as the cumulative gap between the best action on each round and the best policy in $\Pi$. Then $R_T-\operatorname{App}_T(\Pi)$ is regret relative to the best policy in the class.

\begin{theorem}[Best-in-class regret]
\label{thm:main_asb_regret}
Fix $\delta\in(0,1)$. In the setting of Section~\ref{sec:problem}, let $|\Pi|=N$ and $B\in\{1,\ldots,T\}$. Run \ASB\ (Algorithm~\ref{alg:asb}) using the hard-budget rule specified in Appendix~\ref{app:hard_budget_acquisition}, the fixed learning rate $\eta_t\equiv\eta_\star$ from Appendix~\ref{app:main_regret_proof}, and $\varepsilon_t\equiv\min\{1/2,\sqrt{K\log(12NT/\delta)/T}\}$. Then, with probability at least $1-\delta$,
\begin{equation}
\label{eq:main_policy_regret}
  R_T-\operatorname{App}_T(\Pi)
  \le
  \widetilde O\!\left[
    \min\left\{T,
      \sqrt{KT\log(N/\delta)}
      \left(1+\sqrt{T/B}\right)
    \right\}
  \right].
\end{equation}
\end{theorem}

Up to logarithmic factors, Equation~\eqref{eq:main_policy_regret} is the usual finite-policy contextual-bandit scale $\sqrt{KT\log N}$ multiplied by $1+\sqrt{T/B}$. At $B=T$, this is the usual scale with primary-outcome feedback every round; the budget factor grows as $B$ decreases.

The guarantee is end-to-end and does not assume correctness of the fitted bridge, residual score, or surrogate law. Appendix~\ref{app:learned_refinement_efficiency} separately gives a sufficient condition under which the learned post-surrogate refinement reduces the acquisition-dependent contribution to the regret bound.

\subsection{Primary-Outcome Scarcity Lower Bound}

We next lower-bound regret after relaxing the pathwise hard budget to
$\sup_\nu\E_\nu\sum_{t=1}^TQ_t\le B$, where $\nu$ indexes environments in
the lower-bound family. Every pathwise hard-budget learner satisfies this
constraint, so the same lower bound applies to the original pathwise setting.

\begin{theorem}[Regret lower bound under a primary-outcome budget]
\label{thm:audit_budget_lower_bound_main}
There is a universal constant $c_{\mathrm{lb}}>0$ such that, for every
horizon $T$, $N\ge K\ge2$, and $B\in\{1,\ldots,T\}$, there is a
contextual-bandit family in the setting of Section~\ref{sec:problem}, with
$K$ actions and a policy class $\Pi$ of size $N$ satisfying
$\operatorname{App}_T(\Pi)=0$, for which every learner $\mathcal L$ with
$\sup_\nu\E_\nu\sum_{t=1}^TQ_t\le B$ satisfies
\[
  \sup_\nu\E_\nu R_T
  \ge
  c_{\mathrm{lb}}\min\left\{T,
    \sqrt{KT\log N/\log K}
    \sqrt{T/B}
  \right\}.
\]
\end{theorem}

Unlike the PRE-specific separation in Section~\ref{sec:regret_separation}, Theorem~\ref{thm:audit_budget_lower_bound_main} applies to every learner, showing that the $\sqrt{T/B}$ dependence remains even without the PRE restriction.

\section{Empirical Tests of Acquisition Design}
\label{sec:experiments}

Synthetic experiments isolate the acquisition mechanisms while holding the
primary-outcome decision problem fixed. KuaiRec then evaluates the learned
post-surrogate refinement in a data-derived contextual benchmark. Within each
matched comparison, the policy learner and primary-outcome budget are held
fixed. The KuaiRec comparison repeats this design across budgets.

\subsection{Synthetic tests of acquisition structure}

The synthetic study asks three questions within one family: how regret changes
when either acquisition component---decision relevance or residual RMS---is
removed; whether acquisition timing relative to the current surrogate changes
regret; and whether relocating a fixed average amount of residual information
between policy-agreement and policy-disagreement contexts changes regret.

Contexts are divided into a region where the represented policies agree
and a region where they disagree. Across all variants, the
context distribution, policy class, and $P(Y\mid X,A)$ are fixed, with
$\E(Y\mid X,A,S)=\E(Y\mid X,A)$. The surrogate can therefore reveal
differences in residual magnitude without revealing which action is better.

All tests use $K=8$ actions, $N=256$ policies, horizon $T=10^5$, and an
exact primary-outcome budget of $B=3{,}000$. Figure~\ref{fig:controlled_mechanisms}
summarizes the tests; Appendix~\ref{app:controlled_synthetic_family} gives the
data-generating process.

\begin{figure}[!t]
\centering
\includegraphics[width=\linewidth]{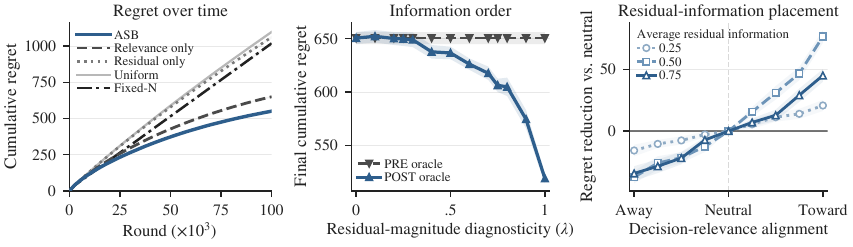}
\caption{\textbf{Synthetic tests of acquisition structure.}
Left: cumulative regret for \ASB, Relevance only, Residual only, Uniform, and
Fixed-N under the same exact primary-outcome budget. Middle: final cumulative
regret for matched PRE and POST oracles as $\lambda$ increases the
surrogate's information about residual magnitude, from none at $\lambda=0$
to exact identification of the high-residual state at $\lambda=1$. Right:
at each fixed average residual-information level, information is shifted from
policy-agreement contexts (Away), through equal information in the two
regions (Neutral), to policy-disagreement contexts (Toward); positive values
mean lower regret than neutral placement. Bands show mean $\pm1$ SE over 64
matched environments.}
\label{fig:controlled_mechanisms}
\end{figure}

\paragraph{Acquisition components.}
The left panel compares \ASB\ with two one-component ablations and two
simple references. \emph{Relevance only} keeps decision relevance but
removes redistribution after observing $S_t$; \emph{Residual only} uses the
learned residual RMS score without decision relevance; \emph{Uniform} uses
neither score; and \emph{Fixed-N} is a budget-matched version of the prefix
strategy in \citet{tucker2023costly}, acquiring the first $B$ primary
outcomes. In this family, \ASB\ has substantially lower cumulative regret
than each of the four alternatives. The comparison shows why both are useful:
decision relevance identifies rounds that distinguish competing policies,
while the residual RMS score prioritizes surrogate states by the estimated
uncertainty remaining after $S_t$.

\paragraph{Acquisition timing.}
The middle panel tests the information-order distinction in
Section~\ref{sec:theory_efficiency} with PRE and POST oracles using exact
residual-RMS quantities around the common fitted bridge. PRE and POST coincide
when the surrogate carries no residual-magnitude information; as $\lambda$
increases, POST attains lower regret. Since
$P(Y\mid X,A)$ is unchanged and $S_t$ does not reveal the better action, the
gap comes from residual-magnitude information being available before
acquisition.

\paragraph{Information placement.}
The right panel moves residual information between the policy-agreement and
policy-disagreement regions while holding its average amount fixed within
each curve. At all three average information levels, placing more of that
information in the disagreement region lowers regret. Thus the same average
amount of residual information is more useful when it lies where represented
policies disagree. The middle of the three
average-information levels allows the largest contrast in residual
information between the agreement and disagreement regions and also shows the
largest regret change.

\subsection{KuaiRec acquisition benchmark}

Using KuaiRec's nearly complete matrix subset \citep{gao2022kuairec}, we
construct a contextual acquisition benchmark with $K=5$ actions. Its
surrogate--primary relationship comes from observed user--video outcomes. On
each round the learner selects one of five candidate videos, observes
$S7=\one\{\text{play duration}>7\text{ seconds}\}$ as a free surrogate, and
then decides whether to acquire
$Y=\min\{\text{watch ratio},1\}$ as the primary outcome. Both regret and terminal policy value are evaluated using the watch-ratio
primary outcome.

The benchmark has $N=9$ represented policies and horizon $T=22{,}417$. We
evaluate primary-outcome budgets of $2,3,4,6,8,12,16,$ and $24\%$. At each budget, \ASB\ is compared with the matched \emph{Relevance only}
reference, which removes only the learned redistribution after observing
$S_t$. The comparison therefore isolates the contribution of the learned
post-surrogate refinement as the primary-outcome budget varies.

The bridge, residual score, and conditional surrogate law are fitted on a
disjoint KuaiRec subset and frozen before evaluation; only acquired benchmark
primary outcomes enter the online learner.
Appendix~\ref{app:kuairec_benchmark} gives the full benchmark construction
and evaluation protocol.

\begin{figure}[!t]
\centering
\includegraphics[width=\linewidth]{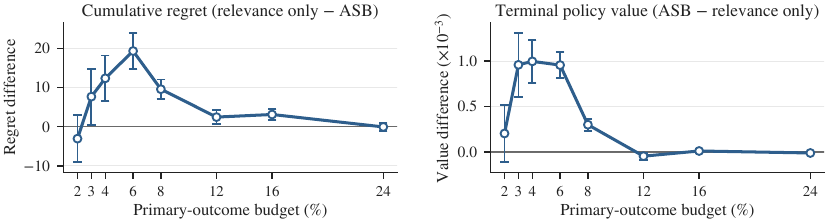}
\caption{\textbf{Post-surrogate refinement on KuaiRec.}
Left: paired cumulative-regret difference, Relevance only minus \ASB.
Right: paired exact terminal policy-value difference, \ASB\ minus Relevance
only, evaluated after the online run using the complete benchmark table. Positive values favor \ASB\ in both panels. Error bars show
mean $\pm1$ SE over 32 paired order/action randomizations and 8 acquisition
randomizations per order/action stream.}
\end{figure}

The clearest separation occurs at $4$--$8\%$, where both mean contrasts
favor \ASB; the differences are small at the lowest budget and again at the
upper end of the evaluated range. Cumulative regret summarizes the learning trajectory, whereas terminal
policy value evaluates the final mixture. The regret contrast remains positive
at budgets where the terminal-value contrast is near zero, so learning-path
gains need not persist in terminal value.

At the low end, full-acquisition initialization consumes most of the budget,
leaving few later rounds on which post-surrogate refinement can change what
is acquired. Toward the high end, both methods acquire many more primary
outcomes and the contrasts narrow. Between these limits, enough
post-initialization rounds remain for refinement to change which outcomes are
acquired, while primary feedback remains limited enough that those choices
affect learning.

\section{Discussion}

In this setting, the ordering matters: the surrogate arrives after the action
but before the primary outcome is measured. It cannot revise $A_t$, but it can
still affect which primary outcomes are observed and used in later policy
updates. This leaves the primary outcome in its original role---defining
action value and regret---while giving the surrogate a separate role in
measurement. That acquisition question is concrete: would measuring the current outcome
distinguish competing policies on this round, and how much uncertainty remains
after $S_t$?

Once primary measurement is selective, three separate questions arise.
Selective measurement can bias an uncorrected update; the augmented feedback
preserves the primary target without requiring a correct bridge. Once that
target is protected, the timing question remains: seeing the current surrogate
before measurement can leave the minimum comparison variance unchanged in
some settings and change the attainable regret rate in others, depending on
whether $S_t$ distinguishes residual uncertainty on rounds that matter to the
policy comparison.
A separate limitation comes from primary-outcome scarcity itself. The lower
bound shows a worst-case $\sqrt{T/B}$ penalty, up to logarithmic factors, even
after the pathwise hard budget is relaxed to an expected acquisition budget.

The next difficulty is learning the acquisition rule online: the learner does
not know in advance which rounds are worth measuring. The selectively acquired primary outcomes feed
both the policy updates and the bridge and residual fits, so evidence collected
for action learning also changes which later outcomes are acquired. \ASB's
worst-case guarantee does not require the fitted bridge, residual score, or
surrogate law to be correct; errors in them can hurt how efficiently the
budget is used without changing the primary target preserved by the augmented
update. The synthetic experiments show why both decision relevance and post-surrogate
residual information matter: in this family, removing either one increases
regret, and the same average residual information is more useful when it is
concentrated where represented policies disagree. KuaiRec then shows how a
learned refinement behaves across finite budgets: the separation from
relevance-only acquisition is clearest at $4$--$8\%$; at the low end,
initialization leaves little room for later refinement, while at the high end
both methods observe many more primary outcomes and the contrast narrows.
Cumulative-regret gains can also persist when terminal-value differences are
near zero. Thus, an opportunity to use post-surrogate information is not the
same as being able to learn and exploit that opportunity from limited primary
feedback.

The present formulation treats measurement as passive: acquiring the current
primary outcome reveals it without changing future contexts or outcomes. If
measurement itself changes future state, acquisition becomes part of a
dynamic decision problem. Within the setting studied here, the information
order has a direct consequence: by the time $S_t$ arrives, the action is
fixed, but the decision to measure $Y_t$ is still open. What matters about the
post-action signal is therefore not only what it predicts, but what it can
still change. The surrogate can change which primary outcomes are measured
for future learning, while $Y_t$ continues to define action value and regret.

\section*{Generative AI disclosure}
Generative AI tools were used to assist with manuscript editing for clarity and readability, related-work search and discovery, implementation of Python code for experiments, and checking and refinement of proof sketches and mathematical arguments. All AI-assisted material was reviewed by the authors, who take responsibility for the final content of the paper.

\bibliographystyle{plainnat}
\bibliography{references,selective_audit_references}

\clearpage
\appendix
\section{Theory of Primary-Outcome Acquisition Before and After the Surrogate}

This appendix supplies the supporting theory for Section~\ref{sec:theory_efficiency}.
We first prove the augmented-feedback identity in Equation~\eqref{eq:target_validity},
identify the acquisition-dependent part of the fixed-comparison variance, and
prove Theorem~\ref{thm:timing_value}. We then prove Theorem~\ref{thm:one_shot_separation_main} within a common
two-action family that includes both deterministic and noisy informative
outcomes.

\subsection{Conditional-Mean Preservation}

Recall from Section~\ref{sec:problem} that
$\E(Y_t\mid\calF_{t-1},U_t)=r^\star(X_t,A_t)$ and, after $Z_t$ is observed,
$Q_t\sim\Bern\{\rho_t(Z_t)\}$ with
$Q_t\perp Y_t\mid(\calF_{t-1},Z_t)$.

\begin{proof}[Proof of Equation~\eqref{eq:target_validity}]
For $\rho_t(Z_t)>0$, conditioning on $(\calF_{t-1},Z_t)$ gives
\[
  \E\!\left[
    \frac{Q_t}{\rho_t(Z_t)}\{Y_t-h_t(Z_t)\}
    \middle|\calF_{t-1},Z_t
  \right]
  =
  \E\{Y_t-h_t(Z_t)\mid\calF_{t-1},Z_t\}.
\]
Hence
$\E(\Phi_t\mid\calF_{t-1},Z_t)=\E(Y_t\mid\calF_{t-1},Z_t)$.
Taking conditional expectation given $(\calF_{t-1},U_t)$ yields
$\E(\Phi_t\mid\calF_{t-1},U_t)=r^\star(X_t,A_t)$.
\end{proof}

\ASB\ uses this identity only when $\rho_t(Z_t)>0$. On stochastic rounds,
$\ell_t(A_t)>0$ implies $\rho_t(Z_t)>0$ through relevance support; when
$\ell_t(A_t)=0$, $c_{t,\pi}(A_t)=0$ for every $\pi$. Deterministic
full-acquisition rounds have $\rho_t=1$.

\subsection{Acquisition-Dependent Variance for a Fixed Policy Comparison}

Fix the comparison in Section~\ref{sec:timing}, suppress the round index, and
hold the bridge $h$ fixed. Recall
$\tau_h^2(U,S)=\E[\{Y-h(U,S)\}^2\mid U,S]$.
For $\rho(U,S)>0$, using $Q^2=Q$ and
$\E(Q\mid U,S)=\rho(U,S)$ in the augmented-feedback formula gives
\[
  \E\!\left[\left.(\Phi-1/2)^2\right|U,S\right]
  =
  \frac{\tau_h^2(U,S)}{\rho(U,S)}
  +2\{h(U,S)-1/2\}\E(Y\mid U,S)
  -h^2(U,S)+1/4.
\]
Therefore, with $g$ from Section~\ref{sec:timing},
\[
  \E\!\left([g(U)\{\Phi-1/2\}]^2\right)
  =
  \mathcal V(\rho)
  +\E\!\left[
    g^2(U)\left\{
      2\{h(U,S)-1/2\}\E(Y\mid U,S)
      -h^2(U,S)+1/4
    \right\}
  \right].
\]
The second term is independent of $\rho$, while
$\mathcal V(\rho)$ is exactly Equation~\eqref{eq:acquisition_variance_burden}.
Because the fixed-comparison mean in Section~\ref{sec:timing} is also
independent of $\rho$, $\mathcal V(\rho)$ is the only acquisition-dependent
part of the variance.

\paragraph{Zero-allocation states.}
For the fixed-comparison optimization, extend the integrand in
$\mathcal V(\rho)$ to $\rho(U,S)=0$ by
\[
  \frac{g^2(U)\tau_h^2(U,S)}{\rho(U,S)}
  :=
  \begin{cases}
    0, & |g(U)|\tau_h(U,S)=0,\\
    +\infty, & |g(U)|\tau_h(U,S)>0.
  \end{cases}
\]
On the same states define the comparison term directly as
$g(U)\{\Phi-1/2\}:=g(U)\{h(U,S)-1/2\}$.
If $\mathcal V(\rho)<\infty$, a zero-allocation state must satisfy
$|g(U)|\tau_h(U,S)=0$. When $g(U)\ne0$, this implies
$\tau_h(U,S)=0$ and hence $Y=h(U,S)$ almost surely given $(U,S)$; when
$g(U)=0$, the comparison term is already zero. The extension therefore
preserves both the comparison mean and the second-moment identity above.

\subsection{Proof of the PRE/POST Variance Characterization}

Set $W(U,S):=|g(U)|\tau_h(U,S)$, so under the preceding convention
$\mathcal V(\rho)=\E\{W^2/\rho\}$.

\begin{proof}[Proof of Theorem~\ref{thm:timing_value}]
Every PRE allocation $\rho(U)$ with $\E\rho\le\bar\rho$ is also a feasible
POST allocation. It therefore remains to identify the two minimizers and
compare their minimum values. The theorem assumption $\Pr\{W(U,S)>0\}>0$ implies $\E W>0$. Under the
theorem's feasibility clause, the proportional allocations in
Equation~\eqref{eq:frontier_optimal_allocations} lie in $[0,1]$ almost surely.

\emph{POST.}
With $W^2/\rho:=0$ when $W=\rho=0$, any POST allocation with
$\E\rho\le\bar\rho$ and finite $\mathcal V(\rho)$ satisfies
\[
  (\E W)^2
  =
  \left[\E\left\{\frac{W}{\sqrt\rho}\sqrt\rho\right\}\right]^2
  \le
  \E\left(\frac{W^2}{\rho}\right)\E\rho
  \le
  \bar\rho\,\mathcal V(\rho).
\]
The POST allocation in
Equation~\eqref{eq:frontier_optimal_allocations} has mean $\bar\rho$ and
attains equality, so
\[
  \mathcal V(\rho_{\rm post}^\star)
  =\frac{(\E W)^2}{\bar\rho}.
\]

\emph{PRE.}
For any PRE allocation with $\E\rho\le\bar\rho$ and finite
$\mathcal V(\rho)$,
\[
  \mathcal V(\rho)
  =
  \E\!\left\{
    \frac{\E(W^2\mid U)}{\rho(U)}
  \right\},
\]
with the ratio defined as zero when its numerator and denominator are both
zero. Cauchy--Schwarz gives
\[
  \left[\E\sqrt{\E(W^2\mid U)}\right]^2
  \le
  \E\!\left\{
    \frac{\E(W^2\mid U)}{\rho(U)}
  \right\}\E\rho(U).
\]
The PRE allocation in Equation~\eqref{eq:frontier_optimal_allocations}
attains equality, and hence
\[
  \mathcal V(\rho_{\rm pre}^\star)
  =
  \frac{
    \left[\E\sqrt{\E(W^2\mid U)}\right]^2
  }{\bar\rho}.
\]

\emph{Strictness.}
Because $W\ge0$,
\[
  \E(W\mid U)\le\sqrt{\E(W^2\mid U)},
\]
with equality almost surely exactly when $\Var(W\mid U)=0$ almost surely.
Since $g(U)$ is fixed given $U$,
\[
  \Var(W\mid U)
  =
  |g(U)|^2\Var\{\tau_h(U,S)\mid U\}.
\]
Taking expectations, squaring, and using the displayed minimum values gives
$\mathcal V(\rho_{\rm post}^\star)\le
\mathcal V(\rho_{\rm pre}^\star)$. The two minimum values are equal exactly when
$\Var\{\tau_h(U,S)\mid U\}=0$ almost surely on $\{|g(U)|>0\}$; otherwise
the POST minimum value is strictly smaller, which is the condition stated in
the theorem.
\end{proof}

\subsection{Regret-Rate Separation on a Common Family}
\label{app:regret_rate_separation_proof}

Extend the family in Section~\ref{sec:regret_separation} by keeping
$X_t\equiv x_0$, $\mathcal A=\{-1,+1\}$, and
$S_t\sim\Bern(q)$ unchanged, and let $\Delta\in(0,1/2]$. Replace the
informative primary outcome by
\[
  Y_t=1/2 \quad\text{if }S_t=0,
  \qquad
  Y_t\mid(A_t,S_t=1)
  \sim
  \Bern\!\left(1/2+\vartheta A_t\Delta\right).
\]
Conditional on the action--surrogate sequence, the primary-outcome draws are
independent across rounds. Moreover,
$\E_\vartheta(Y_t\mid A_t=a)=1/2+\vartheta aq\Delta$, so a wrong action
incurs regret $2q\Delta$. The surrogate law remains independent of
$\vartheta$, and under $h\equiv1/2$ the residual RMS is zero at $S_t=0$ and
$1/2$ at $S_t=1$.

At $\Delta=1/2$, this is exactly the family in
Theorem~\ref{thm:one_shot_separation_main}. Use
$\mathfrak L_{\rm pre}(B)$ as in Section~\ref{sec:regret_separation}, with
expectations taken under the family above.

\begin{proposition}[PRE/POST regret separation on the common family]
\label{prop:common_family_pre_post_separation}
For any $B\in\{1,\ldots,T\}$, every
$\mathcal L\in\mathfrak L_{\rm pre}(B)$ satisfies
\[
  \max_\vartheta\E_\vartheta R_T
  \ge q\Delta T(1-qB)_+.
\]
There is a universal constant $C$ and a bandit strategy with a POST
acquisition rule such that the strategy acquires exactly $B$ primary outcomes
on every path and, for $0<\Delta<1/2$,
\[
  \max_\vartheta\E_\vartheta R_T
  \le C/\Delta+2Tq\Delta e^{-2\Delta^2B}.
\]
For $\Delta=1/2$, the same strategy satisfies
\[
  \max_\vartheta\E_\vartheta R_T
  \le
  [1-(1-q)^T]/2
  \le
  1/2.
\]
\end{proposition}

\begin{proof}
For both parts, let
\[
  K_t:=\sum_{s<t}Q_sS_s
\]
be the number of informative acquisitions before round \(t\).

\emph{PRE lower bound.}
Let \(\Theta\) be uniform on \(\{-1,+1\}\), with observations generated from
environment \(\Theta\). Because a PRE acquisition rule cannot use \(S_t\),
\[
  \E_\vartheta K_{T+1}
  =
  q\,\E_\vartheta\sum_{t=1}^TQ_t
  \le qB
  \qquad\text{for each }\vartheta.
\]
On \(\{K_t=0\}\), the observed history has the same law under the two signs:
the surrogate is sign-independent and every acquired outcome with \(S_t=0\)
equals \(1/2\). Hence
\[
  \Pr(A_t\ne\Theta\mid K_t=0)=1/2,
  \qquad
  \Pr(K_t=0)\ge1-\E K_t\ge(1-qB)_+.
\]
Since a wrong action costs \(2q\Delta\),
\[
  (1/2)\sum_{\vartheta\in\{-1,+1\}}\E_\vartheta R_T
  \ge
  q\Delta\sum_{t=1}^T\Pr(K_t=0)
  \ge
  q\Delta T(1-qB)_+.
\]
The worst-case regret is at least this average.

\emph{POST construction.}
Define the accumulated informative evidence
\[
  H_t:=\sum_{s<t}Q_sS_s\,A_s(2Y_s-1).
\]
Choose \(A_t=\operatorname{sign}(H_t)\), breaking ties uniformly at random.
After observing \(S_t\), set
\[
  Q_t=
  \begin{cases}
    0, & \sum_{s<t}Q_s=B,\\
    1, & B-\sum_{s<t}Q_s=T-t+1,\\
    \one\{S_t=1\}, & \text{otherwise}.
  \end{cases}
\]
Thus the strategy acquires exactly \(B\) outcomes on every path and acquires
every \(S_t=1\) round until \(K_t=B\) or the horizon ends.

\emph{Estimator error.}
Conditional on $K_t=k$, the $k$ nonzero summands in $\vartheta H_t$ are
independent signs, each satisfying
\[
  \Pr_\vartheta\!\left\{\vartheta A_s(2Y_s-1)=1
    \,\middle|\, Q_sS_s=1\right\}
  =1/2+\Delta.
\]
Because $A_t=\operatorname{sign}(H_t)$ with uniform tie-breaking, define
$e_0:=1/2$ and, for $k\ge1$,
\[
  e_k
  :=\Pr_\vartheta(A_t\ne\vartheta\mid K_t=k)
  \le
  \Pr_\vartheta(\vartheta H_t\le0\mid K_t=k)
  \le
  e^{-2\Delta^2k},
\]
where the last inequality is Hoeffding's inequality.

\emph{Regret accumulation.}
While $K_t=k<B$, no new informative term enters $H_t$ until the next
$S_t=1$ round or the horizon, so the action error probability remains $e_k$.
This stage has expected length at most $1/q$ and therefore contributes at most
\[
  (1/q)(2q\Delta)e_k=2\Delta e_k
\]
expected regret. If $K_t$ reaches $B$, the remaining regret is at most
$2Tq\Delta e_B$. Therefore
\[
  \max_\vartheta\E_\vartheta R_T
  \le
  2\Delta\sum_{k=0}^{B-1}e_k
  +2Tq\Delta e_B.
\]
For \(0<\Delta<1/2\), the Hoeffding bound and
\(1-e^{-2\Delta^2}\ge\Delta^2\) give
\[
  2\Delta\sum_{k=0}^{B-1}e_k
  \le
  \Delta+2\Delta/\{1-e^{-2\Delta^2}\}
  \le
  C/\Delta,
\]
which proves the stated noisy upper bound.

When \(\Delta=1/2\), the first informative acquisition identifies
\(\vartheta\) exactly. Before that acquisition \(H_t=0\), so \(A_t\) is
uniform on \(\{-1,+1\}\). Because the first \(S_t=1\) round is always
acquired, the total expected regret through that round is
\[
  (q/2)\sum_{t=1}^T(1-q)^{t-1}
  =
  [1-(1-q)^T]/2.
\]
This proves the \(\Delta=1/2\) bound.
\end{proof}

\paragraph{Theorem~\ref{thm:one_shot_separation_main} as the deterministic special case.}
Set $\Delta=1/2$. Then, whenever $S_t=1$,
$Y_t=\one\{A_t=\vartheta\}$, so the common family is exactly the one in
Section~\ref{sec:regret_separation}. Proposition~\ref{prop:common_family_pre_post_separation} therefore specializes to
Equation~\eqref{eq:one_shot_pre_lower_main} for PRE and
Equation~\eqref{eq:one_shot_post_upper_main} for the exact-$B$ POST strategy.
Setting $q=1/(2B)$ then gives Equation~\eqref{eq:one_shot_rate_gap_main}.

\paragraph{Separation with noisy informative outcomes.}
Set $\Delta=1/4$ and again take $q=1/(2B)$. The informative primary outcome
is then noisy rather than deterministic. The PRE bound becomes $T/(16B)$.
If $B\ge8\log T$, the POST bound is $O(1)$ because
$e^{-2\Delta^2B}=e^{-B/8}\le T^{-1}$. Hence for any budget sequence with $B=o(T)$ and $B\ge8\log T$, every
$\mathcal L\in\mathfrak L_{\rm pre}(B)$ has $\Omega(T/B)$ worst-case
regret while the exact-$B$ POST strategy above has $O(1)$ regret.

\section{Hard-Budget Acquisition for \ASB}
\label{app:hard_budget_acquisition}

Section~\ref{sec:policy_acquisition} specifies how acquisition should vary
with decision relevance before the surrogate and residual RMS after it, but
not how to allocate exactly $B$ primary-outcome acquisitions over a finite
horizon. This appendix supplies the hard-budget rule.

On each round, Algorithm~\ref{alg:hard_budget_pacing} sets the pre-surrogate
level $\rho_t^0$ and the feasible post-surrogate refinement $\rho_t^1$;
Algorithm~\ref{alg:asb} forms the final probability, draws $Q_t$, and updates
the learner.

Deterministic boundary rules enforce exactly $B$ acquisitions. On stochastic
rounds, interval accounting prevents the remaining budget from reaching zero
too early while preserving the support required by the regret analysis.
Section~\ref{app:hard_budget_support_verification} provides the probabilistic
verification.

\paragraph{Initialization and deterministic boundaries.}
Choose $C_{\mathrm{init}}\in\{0,\ldots,B\}$. \ASB\ first acquires every
primary outcome for $C_{\mathrm{init}}$ rounds, setting
$\rho_t^0=\rho_t^1(U_t,\cdot)=1$ for
$t=1,\ldots,C_{\mathrm{init}}$. These acquisitions count toward the same
budget $B$. Before round $t$, write
\[
  B_t^{\mathrm{rem}}:=B-\sum_{i<t}Q_i,
  \qquad
  T_t^{\mathrm{rem}}:=T-t+1.
\]
Outside initialization, no further primary outcomes are acquired when
$B_t^{\mathrm{rem}}=0$. When
$B_t^{\mathrm{rem}}=T_t^{\mathrm{rem}}$, \emph{deterministic final fill}
acquires every remaining primary outcome. All other rounds use the stochastic
branch. Because final fill begins exactly when the remaining acquisition
count equals the remaining horizon, these boundaries give
$\sum_{t=1}^TQ_t=B$ on every path.

\subsection{Pre-surrogate budget allocation}

Section~\ref{sec:policy_acquisition} gives the shape
$\rho_t^0\propto\ell_t(A_t)$. The rule first turns the current ratio of
remaining acquisitions to remaining rounds into an interval target, reserves
part of that target for the support floor defined below, and uses the
remainder for acquisition driven by decision relevance. Past relevance
scores determine its scale.

\paragraph{Interval target.}
Partition the post-initialization rounds into consecutive, disjoint intervals
\[
  \{C_{\mathrm{init}}+1,\ldots,T\}
  = I_1\cup\cdots\cup I_J,
  \qquad
  I_j=\{s_j,\ldots,e_j\},
  \qquad
  L_j=e_j-s_j+1.
\]
At each interval boundary, recompute the target from the current remaining
budget:
\begin{equation}
\label{eq:interval_target_count_app}
  B_j^{\mathrm{tgt}}
  :=
  \begin{cases}
    \displaystyle
    \left\lceil
      \frac{B_{s_j}^{\mathrm{rem}}}{T_{s_j}^{\mathrm{rem}}}L_j
    \right\rceil,
      &j<J,\\[3mm]
    B_{s_j}^{\mathrm{rem}},&j=J.
  \end{cases}
\end{equation}
The final interval simply takes the entire remaining count.

\paragraph{Support floor.}
The regret guarantee requires positive acquisition support whenever decision
relevance is positive, without relying on the fitted residual score. Let
$\bar\rho_T^{\mathrm{supp}}:=B/\{T\log(8T)\}$. On stochastic rounds,
define the \emph{support floor} $\rho_t^{\mathrm{supp}}$ by
\begin{equation}
\label{eq:protected_floor_pacing_app}
  \rho_t^{\mathrm{supp}}
  :=\min\left\{1,
      \bar\rho_T^{\mathrm{supp}}
      \frac{\ell_t(A_t)}{\sum_{a:\,p_t(a)>0} p_t(a)\ell_t(a)}
    \right\},
\end{equation}
with the ratio defined as zero when its denominator is zero. Thus
$\rho_t^{\mathrm{supp}}$ is fixed before $S_t$ and does not use the
learned residual score.

\paragraph{Support reserve and additional budget.}
Even if no probability is added beyond $\rho_t^{\mathrm{supp}}$,
acquisitions still occur with probability $\rho_t^{\mathrm{supp}}$. The
interval target therefore cannot all be used for additional acquisition. We
reserve $B_j^{\mathrm{supp}}$ of that target for these acquisitions.

Fix $\zeta\in(0,1)$. If the reserve were $n$, it would be exceeded when
at least $n+1$ such acquisitions occur in the interval. Let
$\klbin(p\|r):=p\log(p/r)+(1-p)\log\{(1-p)/(1-r)\}$, and choose
$B_j^{\mathrm{supp}}$ as the smallest
$n\in\{0,\ldots,B_j^{\mathrm{tgt}}-1\}$ for which the corresponding
Bernoulli tail bound is at most $\zeta$:
\begin{equation}
\label{eq:support_reserve_app}
  \frac{n+1}{L_j}>\bar\rho_T^{\mathrm{supp}},
  \qquad
  L_j\klbin\!\left(
    \frac{n+1}{L_j}\middle\|\bar\rho_T^{\mathrm{supp}}
  \right)
  \ge\log(1/\zeta),
\end{equation}
if such an integer exists. If the set is empty, set
$B_j^{\mathrm{supp}}=B_j^{\mathrm{tgt}}-1$.
Section~\ref{app:hard_budget_support_verification} verifies this tail bound.

The part of the interval target left for additional acquisition is
\begin{equation}
\label{eq:above_floor_cap_app}
  B_j^{\mathrm{add}}
  :=B_j^{\mathrm{tgt}}-1-B_j^{\mathrm{supp}}.
\end{equation}
For every $j<J$, \eqref{eq:interval_target_count_app} and
\eqref{eq:above_floor_cap_app} give
\[
  B_j^{\mathrm{supp}}+B_j^{\mathrm{add}}
  =B_j^{\mathrm{tgt}}-1
  <
  \frac{B_{s_j}^{\mathrm{rem}}}{T_{s_j}^{\mathrm{rem}}}L_j.
\]
The strict inequality comes from the ceiling in
\eqref{eq:interval_target_count_app}; the $-1$ in
\eqref{eq:above_floor_cap_app} makes the reserve plus additional budget
equal to $B_j^{\mathrm{tgt}}-1$. This leaves one count of slack below the
interval's proportional share of the remaining budget.

\paragraph{Scaling by decision relevance.}
After reserving $B_j^{\mathrm{supp}}$ for the support floor
$\rho_t^{\mathrm{supp}}$, the additional budget
$B_j^{\mathrm{add}}$ is allocated using decision relevance.
Represent a stochastic draw as
$Q_t=\one\{\omega_t<\rho_t(Z_t)\}$ with
$\omega_t\sim{\mathrm{Unif}}(0,1)$, and before round $t\in I_j$ let
\[
  N_{j,t}^{\mathrm{add}}
  :=\sum_{\substack{i\in I_j\\ i<t}}
    \one\{\rho_i^{\mathrm{supp}}\le\omega_i<\rho_i(Z_i)\}
\]
be the number of additional acquisitions already made in $I_j$.

For $t>s_j$, use the rounds already seen in $I_j$ to calibrate the scale
multiplying decision relevance. For a candidate scale $\lambda$, define
\[
  \Psi_{j,t}(\lambda)
  :=\frac{1}{t-s_j}\sum_{i=s_j}^{t-1}
    [\min\{1,\lambda\ell_i(A_i)\}-\rho_i^{\mathrm{supp}}]_+.
\]
Thus $\Psi_{j,t}(\lambda)$ is the average pre-surrogate acquisition
probability that scale $\lambda$ would have added beyond
$\rho_i^{\mathrm{supp}}$ on the earlier rounds of the current interval,
with $\Psi_{j,t}^{\max}:=\max_{\lambda\ge0}\Psi_{j,t}(\lambda)$.

The unused additional budget is
$B_j^{\mathrm{add}}-N_{j,t}^{\mathrm{add}}$, so its target rate over the
remaining rounds is
$(B_j^{\mathrm{add}}-N_{j,t}^{\mathrm{add}})/(e_j-t+1)$. When
$N_{j,t}^{\mathrm{add}}<B_j^{\mathrm{add}}$ and
$\Psi_{j,t}^{\max}>0$, choose the smallest scale whose past average reaches
that rate, capped at the largest attainable past average:
\begin{equation}
\label{eq:empirical_multiplier_pacing_app}
  \lambda_t
  :=\inf\left\{\lambda\ge0:
    \Psi_{j,t}(\lambda)
    \ge
    \min\left\{
      \frac{B_j^{\mathrm{add}}-N_{j,t}^{\mathrm{add}}}{e_j-t+1},
      \Psi_{j,t}^{\max}
    \right\}
  \right\}.
\end{equation}
At an interval start, or if no additional budget remains or
$\Psi_{j,t}^{\max}=0$, retain the previous scale; initialize
$\lambda_{s_1}=0$ for the first post-initialization interval.

The resulting pre-surrogate acquisition level on round $t$ is
\[
  \rho_t^0
  :=
  \begin{cases}
    \max\{\rho_t^{\mathrm{supp}},\min(1,\lambda_t\ell_t(A_t))\},
      &N_{j,t}^{\mathrm{add}}<B_j^{\mathrm{add}},\\
    \rho_t^{\mathrm{supp}},&N_{j,t}^{\mathrm{add}}=B_j^{\mathrm{add}}.
  \end{cases}
\]

\subsection{Post-surrogate refinement and accounting}

By the time $S_t$ is observed, the previous subsection has already fixed
$\rho_t^0$. The post-surrogate refinement from
Section~\ref{sec:policy_acquisition} is then restricted by
\begin{equation}
\label{eq:conditional_refinement_problem}
  \rho_t^{\mathrm{supp}}\le\rho_t^1(U_t,s)\le1,
  \qquad
  \widehat\E_t\{\rho_t^1(U_t,S)\mid U_t\}=\rho_t^0.
\end{equation}

Because $\rho_t^{\mathrm{supp}}\le\rho_t^0\le1$, the constant rule
$\rho_t^1(U_t,\cdot)\equiv\rho_t^0$ is always feasible. While
$N_{j,t}^{\mathrm{add}}<B_j^{\mathrm{add}}$,
Algorithm~\ref{alg:hard_budget_pacing} uses the residual RMS refinement
subject to \eqref{eq:conditional_refinement_problem}; once
$N_{j,t}^{\mathrm{add}}=B_j^{\mathrm{add}}$, it uses the constant rule. With
$\rho_t(U_t,s)=(1-\gamma_t)\rho_t^0+\gamma_t\rho_t^1(U_t,s)$, either
choice preserves
\[
  \widehat\E_t\{\rho_t(U_t,S)\mid U_t\}=\rho_t^0.
\]

The accounting uses the final acquisition probability:
$N_{j,t}^{\mathrm{add}}$ counts acquisitions beyond
$\rho_t^{\mathrm{supp}}$. Once
$N_{j,t}^{\mathrm{add}}=B_j^{\mathrm{add}}$, both
$\rho_t^0$ and $\rho_t^1(U_t,\cdot)$ are set to
$\rho_t^{\mathrm{supp}}$, so the rule adds no further probability beyond
the support floor.

\begin{algorithm}[H]
\caption{Hard-budget acquisition for one \ASB\ round}
\label{alg:hard_budget_pacing}
\algsetup{indent=.8em}
\small
\begin{algorithmic}[1]
\REQUIRE round $t$, horizon $T$, integer budget $B$, initialization length $C_{\mathrm{init}}$, interval partition $\{I_j=\{s_j,\ldots,e_j\}\}_{j=1}^J$, $\zeta\in(0,1)$, action law $p_t$, relevance scores $\ell_t(a)$, fitted residual score $\widehat\tau_t(U_t,\cdot)$, and fitted conditional expectation $\widehat\E_t(\cdot\mid U_t)$.
\STATE Compute $B_t^{\mathrm{rem}}=B-\sum_{i<t}Q_i$ and $T_t^{\mathrm{rem}}=T-t+1$.
\IF{$t\le C_{\mathrm{init}}$ or $B_t^{\mathrm{rem}}=T_t^{\mathrm{rem}}$}
  \STATE Set $\rho_t^0=\rho_t^1(U_t,\cdot)=1$ and mark the round deterministic.
\ELSIF{$B_t^{\mathrm{rem}}=0$}
  \STATE Set $\rho_t^0=\rho_t^1(U_t,\cdot)=0$ and mark the round deterministic.
\ELSE
  \STATE Mark the round stochastic and let $t\in I_j$. If $t=s_j$, compute $B_j^{\mathrm{tgt}}$ by \eqref{eq:interval_target_count_app}, $B_j^{\mathrm{supp}}$ by \eqref{eq:support_reserve_app} with $\zeta$, and $B_j^{\mathrm{add}}$ by \eqref{eq:above_floor_cap_app}.
  \STATE Set $\rho_t^{\mathrm{supp}}$ by \eqref{eq:protected_floor_pacing_app} and compute $N_{j,t}^{\mathrm{add}}$ from the realized earlier rounds in $I_j$.
  \STATE Set $\lambda_t=0$ if $t=s_1$ and $\lambda_t=\lambda_{t-1}$ otherwise.
  \IF{$N_{j,t}^{\mathrm{add}}<B_j^{\mathrm{add}}$ and $t>s_j$}
    \STATE Evaluate $\Psi_{j,t}^{\max}$ from the past rounds and, if positive, set $\lambda_t$ by \eqref{eq:empirical_multiplier_pacing_app}.
  \ENDIF
  \STATE If $N_{j,t}^{\mathrm{add}}<B_j^{\mathrm{add}}$, set $\rho_t^0=\max\{\rho_t^{\mathrm{supp}},\min(1,\lambda_t\ell_t(A_t))\}$; otherwise set $\rho_t^0=\rho_t^{\mathrm{supp}}$.
\ENDIF
\STATE Observe $S_t$.
\IF{the round is stochastic and $N_{j,t}^{\mathrm{add}}<B_j^{\mathrm{add}}$}
  \STATE Compute the residual RMS refinement in \eqref{eq:conditional_refinement_variational_main} subject to \eqref{eq:conditional_refinement_problem}, obtaining $\rho_t^1(Z_t)$.
\ELSIF{the round is stochastic}
  \STATE Set $\rho_t^1(U_t,\cdot)=\rho_t^0=\rho_t^{\mathrm{supp}}$.
\ENDIF
\STATE \textbf{Output:} $\rho_t^0$ and $\rho_t^1(Z_t)$ to Algorithm~\ref{alg:asb}.
\end{algorithmic}
\end{algorithm}

\paragraph{Theorem specialization and reported implementation.}
The hard-budget rule allows any fixed finite partition. For the regret
guarantee, the number of intervals must remain fixed and every interval must
retain a constant fraction of the post-initialization horizon. Call such a
family of partitions \emph{regular} if there is a fixed
\emph{regularity constant} $\kappa>0$, independent of $T$, such that
\[
  L_j\ge \kappa(T-C_{\mathrm{init}})-1
  \qquad\text{for every }j.
\]

For the regret guarantee in Section~\ref{sec:asb_regret_guarantee}, use
$C_{\mathrm{init}}=0$, any fixed regular partition, and
$\zeta=\delta/(2J)$ in each interval. A bounded full-acquisition
initialization is treated separately in
Section~\ref{app:bounded_initialization_extension}, and
Appendix~\ref{app:common_finite_policy_implementation} records the reported
implementation choices.

\section{ASB Regret Analysis}

\paragraph{Theorem proof.}
Fix $\alpha>0$ and assume for now that every round satisfies
\begin{equation}
\label{eq:relevance_support_requirement_app}
  \rho_t(Z_t)\ge
  \min\{1,\alpha\,\ell_t(A_t)\}.
\end{equation}
Section~\ref{app:gain_second_moment_control} controls the gain range and
policy-mixture second moment under
\eqref{eq:relevance_support_requirement_app}.
Section~\ref{app:policy_update_regret} converts those bounds into a fixed-rate
regret bound. Section~\ref{app:hard_budget_support_verification} then shows
that Algorithm~\ref{alg:hard_budget_pacing} supplies the required support
with $\alpha=B/\{KT\log(8T)\}$ with high probability, and
Section~\ref{app:main_regret_proof} combines the three ingredients to prove
Theorem~\ref{thm:main_asb_regret}.

\paragraph{Additional results.}
Section~\ref{app:bounded_initialization_extension} shows that the optional
full-acquisition initialization allowed by
Algorithm~\ref{alg:hard_budget_pacing} preserves the same regret rate when
its length is controlled. Section~\ref{app:learned_refinement_efficiency}
shows when the learned post-surrogate redistribution lowers the
acquisition-dependent term $\mathsf D_T$ in the regret bound. Neither result
is used to prove Theorem~\ref{thm:main_asb_regret}.

Throughout this appendix, write
$\E_t(\cdot):=\E(\cdot\mid\calF_{t-1},X_t)$.
Action sums involving $c_{t,\pi}(a)$ or $\ell_t(a)$ are over actions with
$p_t(a)>0$.

\subsection{Gain range and second-moment bounds}
\label{app:gain_second_moment_control}

Under the generic support condition \eqref{eq:relevance_support_requirement_app}, we first control the gain range and the policy-mixture second moment used by the fixed-rate learner analysis.

Let $\underline\varepsilon>0$ be a common lower bound on the exploration rates, with $\varepsilon_t\le1/2$. Since $u_t$ is uniform on the support of $\bar p_t$, $p_t(a)\ge\underline\varepsilon/K$ there, so \eqref{eq:centered_policy_influence} gives $|c_{t,\pi}(a)|\le\ell_t(a)\le K/\underline\varepsilon$. Write
$G_{\mathrm{bd}}:=K/(2\underline\varepsilon)+1/\alpha$.

If $\ell_t(A_t)=0$, then $c_{t,\pi}(A_t)=0$ for every policy and hence $G_{t,\pi}=0$; Algorithm~\ref{alg:asb} likewise sets $G_{t,\pi}=0$ when $\rho_t=0$. If $0<\alpha\ell_t(A_t)<1$, \eqref{eq:relevance_support_requirement_app} gives
$\rho_t\ge\alpha\ell_t(A_t)$, so
\[
 |G_{t,\pi}|
 \le \ell_t(A_t)/2+|c_{t,\pi}(A_t)|/\rho_t
 \le K/(2\underline\varepsilon)+1/\alpha.
\]
If $\alpha\ell_t(A_t)\ge1$, the same support condition forces $\rho_t=1$, hence $\Phi_t=Y_t$ and $|G_{t,\pi}|\le\ell_t(A_t)/2$. Thus $|G_{t,\pi}|\le G_{\mathrm{bd}}$ on every round satisfying \eqref{eq:relevance_support_requirement_app}.

For the mixture second moment, write
$\overline c_{\mu,t}^2(a):=\sum_{\pi}\mu_t(\pi)c_{t,\pi}^2(a)$ and define
\begin{equation}
\label{eq:cumulative_acquisition_second_moment_app}
  d_t
  =\E_t\left[
    \overline c_{\mu,t}^2(A_t)\tau_t^2(Z_t)/\rho_t(Z_t)
  \right],
  \qquad
  \mathsf D_T=\sum_{t=1}^T d_t.
\end{equation}

When $\rho_t(Z_t)=0$, \eqref{eq:relevance_support_requirement_app} forces $\ell_t(A_t)=0$ and hence $\overline c_{\mu,t}^2(A_t)=0$; define the corresponding ratio in \eqref{eq:cumulative_acquisition_second_moment_app} as zero. Thus $d_t$ is the acquisition-dependent part of the policy-mixture second moment on round $t$, averaged over the current action and post-action randomness, and $\mathsf D_T$ is its cumulative version.

\begin{lemma}[Policy-mixture second-moment bound]
\label{lem:policy_mixture_second_moment}
If \eqref{eq:relevance_support_requirement_app} holds on every round, then
\[
  \E_t\sum_{\pi\in\Pi}\mu_t(\pi)G_{t,\pi}^2
  \le K+2d_t\quad\text{for all }t,
  \qquad
  \mathsf D_T\le 2KT+T/\alpha.
\]
\end{lemma}

\begin{proof}
From \eqref{eq:centered_policy_influence},
$\overline c_{\mu,t}^2(a)
=\bar p_t(a)\{1-\bar p_t(a)\}/p_t^2(a)$.
Because $p_t(a)\ge(1-\varepsilon_t)\bar p_t(a)$ and
$\varepsilon_t\le1/2$,
\[
  \E_t\{\overline c_{\mu,t}^2(A_t)\}
  =\sum_a p_t(a)\overline c_{\mu,t}^2(a)
  \le2K.
\]
Also, for $\ell_t(a)>0$,
$\overline c_{\mu,t}^2(a)/\ell_t(a)
\le\min\{\bar p_t(a),1-\bar p_t(a)\}/p_t(a)$.
Defining the ratio as zero when $\ell_t(a)=0$ gives
\[
  \E_t\!\left\{\overline c_{\mu,t}^2(A_t)/\ell_t(A_t)\right\}
  =\sum_a p_t(a)\,\overline c_{\mu,t}^2(a)/\ell_t(a)
  \le1.
\]

On states with $\rho_t(Z_t)=0$, all policy coefficients vanish, so the contribution to $\sum_\pi\mu_t(\pi)G_{t,\pi}^2$ is zero. On the remaining states, conditioning on $(\calF_{t-1},Z_t)$ and using
$(x+y)^2\le2x^2+2y^2$, $h_t\in[0,1]$, $|Y_t-h_t(Z_t)|\le1$, and
\[
  \E\!\left[
    \left.
    \frac{Q_t\{Y_t-h_t(Z_t)\}^2}{\rho_t^2(Z_t)}
    \right|\calF_{t-1},Z_t
  \right]
  =\tau_t^2(Z_t)/\rho_t(Z_t)
\]
gives
\[
  \E\!\left[
    \left.
    \left(\Phi_t-1/2\right)^2
    \right|\calF_{t-1},Z_t
  \right]
  \le
  1/2+2\tau_t^2(Z_t)/\rho_t(Z_t).
\]
Therefore
\[
\begin{aligned}
  \E_t\sum_{\pi}\mu_t(\pi)G_{t,\pi}^2
  &=\E_t\!\left[
    \one\{\rho_t(Z_t)>0\}\,
    \overline c_{\mu,t}^2(A_t)
    \left(\Phi_t-1/2\right)^2
  \right]\\
  &\le (1/2)\E_t\{\overline c_{\mu,t}^2(A_t)\}+2d_t
  \le K+2d_t.
\end{aligned}
\]

Finally, \eqref{eq:relevance_support_requirement_app} and
$1/\min(1,x)\le1+1/x$ give
\[
  d_t
  \le \E_t\{\overline c_{\mu,t}^2(A_t)\tau_t^2(Z_t)\}
  +(1/\alpha)
   \E_t\!\left\{\overline c_{\mu,t}^2(A_t)/\ell_t(A_t)\right\}
  \le 2K+1/\alpha.
\]
Summing over $t$ yields $\mathsf D_T\le2KT+T/\alpha$.
\end{proof}

\subsection{Fixed-rate regret bound}
\label{app:policy_update_regret}

For the confidence level $\delta$ in Theorem~\ref{thm:main_asb_regret}, write
$L_T:=\log(12NT/\delta)$. The next lemma converts the gain-range and policy-mixture second-moment bounds above into a regret bound for a fixed learning rate.

\begin{lemma}[Policy regret for a fixed learning rate]
\label{lem:policywise_regret_bound}
Follow Algorithm~\ref{alg:asb} for $\mu_t$, $p_t$, $A_t$,
$G_{t,\pi}$, and the weight update
\eqref{eq:observed_second_order_update}. The acquisition rule may otherwise
be arbitrary, provided it satisfies the timing and conditional-independence
conditions of Section~\ref{sec:problem} and
\eqref{eq:relevance_support_requirement_app} on every round.
For any fixed $0<\eta\le1/(2G_{\mathrm{bd}})$, set
$\eta_t\equiv\eta$. Then there is a universal constant $C$ such that,
with probability at least $1-\delta/2$, simultaneously for all $\pi\in\Pi$,
\begin{equation}
\label{eq:fixed_rate_policy_regret_app}
\begin{aligned}
  \sum_{t=1}^T
  \left[
    r^\star\{X_t,\pi(X_t)\}-r^\star(X_t,A_t)
  \right]
  &\le
  \sum_{t=1}^T\varepsilon_t
  +C\sqrt{TL_T}
  +2L_T/\eta\\
  &\quad
  +C\eta\{KT+\mathsf D_T\}
  +C G_{\mathrm{bd}}L_T.
\end{aligned}
\end{equation}
\end{lemma}

\begin{proof}
Write
\[
  v_t^G:=\sum_{\pi}\mu_t(\pi)G_{t,\pi}^2,
  \qquad
  V_T^G:=\sum_{t=1}^T v_t^G.
\]

\emph{Weight-potential bound.}
By
\eqref{eq:centered_policy_influence},
$\sum_\pi\mu_t(\pi)c_{t,\pi}(a)=0$ for every action $a$, hence
$\sum_\pi\mu_t(\pi)G_{t,\pi}=0$ pathwise. Let
$W_t:=\sum_\pi w_t(\pi)$. By the gain-range bound above and $\eta\le1/(2G_{\mathrm{bd}})$,
$|\eta G_{t,\pi}|\le1/2$. Hence the update
\eqref{eq:observed_second_order_update} and
$\exp(x+x^2)\le1+x+3x^2$ give
\[
  W_{t+1}/W_t
  =\sum_\pi\mu_t(\pi)
    \exp\{\eta G_{t,\pi}+\eta^2G_{t,\pi}^2\}
  \le 1+3\eta^2v_t^G.
\]

Taking logs and summing over $t$ gives
$\log(W_{T+1}/W_1)\le3\eta^2V_T^G$.
For any $\pi\in\Pi$, the same update gives
\[
\begin{aligned}
  \eta\sum_{t=1}^T G_{t,\pi}
  +\eta^2\sum_{t=1}^T G_{t,\pi}^2
  &=\log\{w_{T+1}(\pi)/w_1(\pi)\}\\
  &\le
    \log(W_{T+1}/W_1)
    +\log\{W_1/w_1(\pi)\}
  \le 3\eta^2V_T^G+\log N,
\end{aligned}
\]
where the last term uses the equal-weight initialization in
Algorithm~\ref{alg:asb}. Rearranging,
\[
  \sum_tG_{t,\pi}
  \le
  (\log N)/\eta
  +3\eta V_T^G
  -\eta\sum_tG_{t,\pi}^2.
\]

\emph{Conditional-comparison bound.}
For a fixed $\pi$, set
\[
  M_{t,\pi}
  :=
  \exp\!\left\{
    \eta\sum_{i=1}^t(\E_iG_{i,\pi}-G_{i,\pi})
    -\eta^2\sum_{i=1}^tG_{i,\pi}^2
  \right\},
  \qquad M_{0,\pi}=1.
\]
Writing $m_t:=\eta\E_tG_{t,\pi}$, so $|m_t|\le1/2$, its one-step conditional mean satisfies
\[
\begin{aligned}
  \E_t\!\left(M_{t,\pi}/M_{t-1,\pi}\right)
  &=
  \exp(m_t)\,
  \E_t\exp\{-\eta G_{t,\pi}-\eta^2G_{t,\pi}^2\}\\
  &\le \exp(m_t)(1-m_t)
  \le1.
\end{aligned}
\]
Here the first inequality uses
$\exp(-x-x^2)\le1-x$ for $|x|\le1/2$, and the second uses
$\exp(m)(1-m)\le1$ for $m<1$.
Thus $(M_{t,\pi})_{t=0}^T$ is a nonnegative supermartingale over rounds.

Markov's inequality at $T$ and a union bound over $\Pi$ give, with failure
probability at most $\delta/6$,
\[
  \sum_t\E_tG_{t,\pi}
  \le
  \sum_tG_{t,\pi}
  +\eta\sum_tG_{t,\pi}^2
  +\log(6N/\delta)/\eta
\]
simultaneously for all $\pi\in\Pi$. Denote this event by
$\mathcal E_{\rm gain}$. Substituting the preceding bound on
$\sum_tG_{t,\pi}$ cancels the two terms
$\pm\eta\sum_tG_{t,\pi}^2$; since
$\log N+\log(6N/\delta)\le2L_T$, on $\mathcal E_{\rm gain}$,
\[
  \sum_t\E_tG_{t,\pi}
  \le
  2L_T/\eta
  +3\eta V_T^G
  \qquad\text{for all }\pi\in\Pi.
\]

\emph{Mixture-second-moment concentration.}
We next control $V_T^G$.
Lemma~\ref{lem:policy_mixture_second_moment} gives
$\E_tv_t^G\le K+2d_t$, while
$0\le v_t^G\le G_{\mathrm{bd}}^2$.
Apply $\exp(y)\le1+(e-1)y$ with
$y=v_t^G/G_{\mathrm{bd}}^2\in[0,1]$, take $\E_t$, and use
$1+x\le\exp(x)$. This gives
\[
  \E_t\exp\!\left\{
    \{v_t^G-(e-1)\E_tv_t^G\}/G_{\mathrm{bd}}^2
  \right\}
  \le1.
\]

Iterating this one-step bound and applying Markov's inequality gives, with
failure probability at most $\delta/6$,
\[
\begin{aligned}
  V_T^G
  &\le
  (e-1)\sum_{t=1}^T\E_tv_t^G
  +G_{\mathrm{bd}}^2\log(6/\delta)\\
  &\le
  (e-1)\{KT+2\mathsf D_T\}
  +G_{\mathrm{bd}}^2\log(6/\delta).
\end{aligned}
\]
Denote this event by $\mathcal E_{\rm mix}$. Substituting this bound into
the inequality on $\mathcal E_{\rm gain}$ and using
$\eta G_{\mathrm{bd}}\le1/2$ and
$\log(6/\delta)\le L_T$ gives
\[
  \sum_t\E_tG_{t,\pi}
  \le
  2L_T/\eta
  +C\eta\{KT+\mathsf D_T\}
  +C G_{\mathrm{bd}}L_T
  \qquad\text{for all }\pi\in\Pi
\]
on $\mathcal E_{\rm gain}\cap\mathcal E_{\rm mix}$.

\emph{Action-regret conversion.}
To pass from the conditional comparison to realized action regret, let
$D_{t,\pi}^{\mathrm{act}}
:=r^\star\{X_t,\pi(X_t)\}-r^\star(X_t,A_t)$.
Since $A_t\sim p_t=(1-\varepsilon_t)\bar p_t+\varepsilon_tu_t$ and
$r^\star\in[0,1]$,
\[
\begin{aligned}
  \E_tD_{t,\pi}^{\mathrm{act}}
  &=
  r^\star\{X_t,\pi(X_t)\}
  -\sum_a p_t(a)r^\star(X_t,a)\\
  &\le
  r^\star\{X_t,\pi(X_t)\}
  -\sum_a\bar p_t(a)r^\star(X_t,a)
  +\varepsilon_t
  =\E_tG_{t,\pi}+\varepsilon_t,
\end{aligned}
\]
where the last equality is \eqref{eq:policy_comparison_target_main}.

Because $D_{t,\pi}^{\mathrm{act}}\in[-1,1]$, the centered increments are
bounded by two. Hoeffding--Azuma and a union bound over the $N$ policies,
with $\log(6N/\delta)\le L_T$, therefore give, with failure probability
at most $\delta/6$,
\[
  \sum_t
  \left(D_{t,\pi}^{\mathrm{act}}-\E_tD_{t,\pi}^{\mathrm{act}}\right)
  \le C\sqrt{TL_T}
\]
simultaneously for all $\pi\in\Pi$. Denote this event by
$\mathcal E_{\rm act}$. By a union bound,
\[
  \Pr\!\left(
    \mathcal E_{\rm gain}\cap
    \mathcal E_{\rm mix}\cap
    \mathcal E_{\rm act}
  \right)
  \ge1-\delta/2.
\]
On this intersection, the preceding inequalities give
\eqref{eq:fixed_rate_policy_regret_app}.
\end{proof}

\subsection{Hard-budget support verification}
\label{app:hard_budget_support_verification}

The preceding analysis is generic in $\alpha$. We now use the theorem
specialization of Algorithm~\ref{alg:hard_budget_pacing} from
Appendix~\ref{app:hard_budget_acquisition}: $C_{\mathrm{init}}=0$, a fixed
regular interval partition with regularity constant $\kappa>0$, and
$\zeta=\delta/(2J)$. We verify
\eqref{eq:relevance_support_requirement_app} with
$\alpha=B/\{KT\log(8T)\}$.

\begin{lemma}[Hard-budget acquisition support]
\label{lem:hard_budget_acquisition_lower_bound}
There is a constant
$C_{\mathrm{pace}}=C_{\mathrm{pace}}(J,\kappa)$ such that, whenever
$B\ge C_{\mathrm{pace}}K L_T$, an event $\mathcal E_{\mathrm{supp}}$ of
probability at least $1-\delta/2$ satisfies
\[
  \rho_t(Z_t)
  \ge
  \min\left\{
    1,B\ell_t(A_t)/\{KT\log(8T)\}
  \right\}
  \qquad\text{for all }t.
\]
\end{lemma}

\begin{proof}
Recall from Appendix~\ref{app:hard_budget_acquisition} that
$B_t^{\mathrm{rem}}=B-\sum_{i<t}Q_i$ is the acquisition count still
available before round $t$, while
$T_t^{\mathrm{rem}}=T-t+1$ is the number of rounds remaining.
Algorithm~\ref{alg:hard_budget_pacing} uses the stochastic branch when
$0<B_t^{\mathrm{rem}}<T_t^{\mathrm{rem}}$, starts deterministic final fill
when $B_t^{\mathrm{rem}}=T_t^{\mathrm{rem}}$, and sets $\rho_t=0$ when
$B_t^{\mathrm{rem}}=0$.

On a stochastic round, the support floor from
Appendix~\ref{app:hard_budget_acquisition} uses
$\bar\rho_T^{\mathrm{supp}}=B/\{T\log(8T)\}$.
By definition,
$p_t(a)\ell_t(a)
=\max_{\pi'\in\Pi}|\one\{\pi'(X_t)=a\}-\bar p_t(a)|\le1$
for every action with $p_t(a)>0$, so
$\sum_a p_t(a)\ell_t(a)\le K$ and therefore
\[
  \rho_t^{\mathrm{supp}}
  \ge \min\left\{1,
    B\ell_t(A_t)/\{KT\log(8T)\}
  \right\}.
\]

Moreover,
$\rho_t^0,\rho_t^1(Z_t)\ge\rho_t^{\mathrm{supp}}$, so
\[
  \rho_t(Z_t)
  =(1-\gamma_t)\rho_t^0+\gamma_t\rho_t^1(Z_t)
  \ge\rho_t^{\mathrm{supp}}.
\]
If instead $B_t^{\mathrm{rem}}=T_t^{\mathrm{rem}}$, then $\rho_t=1$.
Consequently, \eqref{eq:relevance_support_requirement_app} holds on every
round with $B_t^{\mathrm{rem}}>0$. It remains only to prove
\[
  \Pr\!\left(
    B_t^{\mathrm{rem}}>0
    \text{ for every }t=1,\ldots,T
  \right)
  \ge1-\delta/2.
\]

\emph{Interval budget accounting.}
Consider an interval $I_j$ with
$0<B_{s_j}^{\mathrm{rem}}<T_{s_j}^{\mathrm{rem}}$, so its first round
uses the stochastic branch.
Using the uniform representation of the acquisition draw from
Appendix~\ref{app:hard_budget_acquisition}, define the number of
acquisitions attributed to the support floor in the interval by
\[
  N_j^{\mathrm{floor}}
  :=
  \sum_{t\in I_j}
  \one\{\text{$t$ is stochastic}\}
  \one\{\omega_t<\rho_t^{\mathrm{supp}}\},
  \qquad
  \mathcal E_j
  :=
  \{N_j^{\mathrm{floor}}\le B_j^{\mathrm{supp}}\}.
\]

By \eqref{eq:above_floor_cap_app}, at most
$B_j^{\mathrm{add}}=B_j^{\mathrm{tgt}}-1-B_j^{\mathrm{supp}}$
additional acquisitions can occur beyond $\rho_t^{\mathrm{supp}}$.
Hence, for any $u\in I_j$
before deterministic final fill begins, on $\mathcal E_j$,
\[
  \sum_{\substack{t\in I_j\\ t\le u}}Q_t
  \le
  N_j^{\mathrm{floor}}+B_j^{\mathrm{add}}
  \le
  B_j^{\mathrm{tgt}}-1
  \le
  B_{s_j}^{\mathrm{rem}}-1.
\]
Thus, on $\mathcal E_j$, the interval cannot reduce the remaining budget
to zero before deterministic final fill.

\emph{Remaining budget ratio.}
Write
\[
  q:=B/T,
  \qquad
  q_j:=B_{s_j}^{\mathrm{rem}}/T_{s_j}^{\mathrm{rem}}.
\]
For an $I_j$ with $j<J$ that completes before deterministic final fill,
$B_j^{\mathrm{tgt}}=\lceil q_jL_j\rceil$ and the preceding bound give
\[
\begin{aligned}
  \sum_{t\in I_j}Q_t
  &\le B_j^{\mathrm{tgt}}-1<q_jL_j,\\
  B_{s_{j+1}}^{\mathrm{rem}}
  &=B_{s_j}^{\mathrm{rem}}-\sum_{t\in I_j}Q_t
    >q_j\{T_{s_j}^{\mathrm{rem}}-L_j\}
    =q_jT_{s_{j+1}}^{\mathrm{rem}}.
\end{aligned}
\]
Thus $q_{j+1}>q_j$, and on
$\mathcal E_1\cap\cdots\cap\mathcal E_{j-1}$ we have $q_j\ge q_1=q$.

\emph{Concentration of $N_j^{\mathrm{floor}}$.}
By \eqref{eq:protected_floor_pacing_app},
$\E_t\rho_t^{\mathrm{supp}}\le\bar\rho_T^{\mathrm{supp}}$: when the
normalizing denominator is positive, the expectation of the untruncated
ratio is $\bar\rho_T^{\mathrm{supp}}$, and when it is zero the floor is
zero. Thus each summand in $N_j^{\mathrm{floor}}$ has conditional mean at
most $\bar\rho_T^{\mathrm{supp}}$. Iterating the corresponding one-step
Bernoulli mgf bound and applying Chernoff gives, for
$m/L_j>\bar\rho_T^{\mathrm{supp}}$,
\[
  \Pr\!\left(
    N_j^{\mathrm{floor}}\ge m
    \mid\calF_{s_j-1}
  \right)
  \le
  \exp\!\left\{
    -L_j\klbin\!\left(
      m/L_j\middle\|\bar\rho_T^{\mathrm{supp}}
    \right)
  \right\}.
\]

We next verify that the threshold set defining
$B_j^{\mathrm{supp}}$ in \eqref{eq:support_reserve_app} is nonempty.
Choose $C_{\mathrm{pace}}(J,\kappa)$ once so that
$B\ge C_{\mathrm{pace}}K L_T$ implies
\[
  T\ge2/\kappa,
  \qquad
  \kappa B/10\ge\log(2J/\delta).
\]
Because $C_{\mathrm{init}}=0$, regularity gives
$L_j\ge\kappa T-1$; since $T\ge2/\kappa$, this implies
$L_j\ge\kappa T/2$. On
$\mathcal E_1\cap\cdots\cap\mathcal E_{j-1}$, the preceding step gives
$q_j\ge q$. For the candidate $n=B_j^{\mathrm{tgt}}-1$,
$(n+1)/L_j=B_j^{\mathrm{tgt}}/L_j\ge q_j\ge q
>q/\log(8T)=\bar\rho_T^{\mathrm{supp}}$, which verifies the first
inequality in \eqref{eq:support_reserve_app}.

For the KL inequality, the first step below uses that
$p\mapsto\klbin(p\|r)$ is increasing for $p>r$. Expanding the Bernoulli
KL and using
$(1-q)\log\{(1-q)/(1-q/\log(8T))\}\ge q/\log(8T)-q$
gives the next step.

Finally, $L_j\ge\kappa T/2$, $q=B/T$, and
$\log\log(8T)-1+1/\log(8T)\ge1/5$ give the last:
\[
\begin{aligned}
  L_j\klbin\!\left(
    B_j^{\mathrm{tgt}}/L_j\middle\|q/\log(8T)
  \right)
  &\ge L_j\klbin\!\left(q\middle\|q/\log(8T)\right)\\
  &\ge L_j q\{\log\log(8T)-1+1/\log(8T)\}\\
  &\ge \kappa B/10
  \ge \log(1/\zeta).
\end{aligned}
\]
Thus \eqref{eq:support_reserve_app} is satisfied at
$n=B_j^{\mathrm{tgt}}-1$, so its defining set is nonempty.

Since
$\mathcal E_j^c=\{N_j^{\mathrm{floor}}\ge B_j^{\mathrm{supp}}+1\}$,
the Chernoff bound at $m=B_j^{\mathrm{supp}}+1$ gives
\[
  \Pr\!\left(
    \mathcal E_j^c
    \mid\calF_{s_j-1}
  \right)
  \le\zeta
\]
whenever the interval starts on the stochastic branch and the preceding
interval events hold.

\emph{Union bound across intervals.}
Let $\mathcal E_{\mathrm{supp}}$ be the event that every interval beginning
before deterministic final fill satisfies its $\mathcal E_j$ event.
At most $J$ such events are needed, so a sequential union bound gives
\[
  \Pr(\mathcal E_{\mathrm{supp}})
  \ge1-J\zeta
  =1-\delta/2.
\]

On $\mathcal E_{\mathrm{supp}}$, the interval accounting bound prevents
$B_t^{\mathrm{rem}}$ from reaching zero before deterministic final fill;
after final fill begins,
$B_t^{\mathrm{rem}}=T_t^{\mathrm{rem}}>0$. Therefore
\[
  \mathcal E_{\mathrm{supp}}
  \subseteq
  \left\{
    B_t^{\mathrm{rem}}>0
    \text{ for every }t=1,\ldots,T
  \right\}.
\]
Combining this inclusion with the support bound established at the start
of the proof proves the lemma.
\end{proof}

\subsection{Proof of the ASB regret guarantee}
\label{app:main_regret_proof}

Sections~\ref{app:gain_second_moment_control}--\ref{app:policy_update_regret}
give a regret bound under the generic support condition
\eqref{eq:relevance_support_requirement_app}. Section~\ref{app:hard_budget_support_verification}
shows that, whenever $B\ge C_{\mathrm{pace}}K L_T$, Algorithm~\ref{alg:hard_budget_pacing}
satisfies this condition with
$\alpha=B/\{KT\log(8T)\}$ on an event of probability at least
$1-\delta/2$.

\emph{Small-budget case.}
If $B<C_{\mathrm{pace}}K L_T$, then
\[
  T\sqrt{K L_T/B}
  >T/\sqrt{C_{\mathrm{pace}}},
\]
so the deterministic bound
$R_T-\operatorname{App}_T(\Pi)\le T$ is sufficient for the stated
$\widetilde O$ rate. Hence assume $B\ge C_{\mathrm{pace}}K L_T$, set
$\alpha=B/\{KT\log(8T)\}$, and let $\mathcal E_{\mathrm{supp}}$ be
the event from Lemma~\ref{lem:hard_budget_acquisition_lower_bound}.

\emph{Applying the fixed-rate regret bound.}
Lemma~\ref{lem:policywise_regret_bound} is proved for a process satisfying
\eqref{eq:relevance_support_requirement_app} on every round. Lemma~\ref{lem:hard_budget_acquisition_lower_bound}
gives this property for ASB only on $\mathcal E_{\mathrm{supp}}$. We therefore
alter ASB only after the first support failure, and only for the purpose of
applying Lemma~\ref{lem:policywise_regret_bound}.

Define
\[
  \sigma
  :=\inf\left\{
    t\le T:
    \rho_t(Z_t)
    <
    \min\{1,\alpha\ell_t(A_t)\}
  \right\},
\]
with $\sigma=T+1$ if no such round occurs, and set
\[
  \widetilde\rho_t(Z_t)
  :=
  \begin{cases}
    \rho_t(Z_t),&t<\sigma,\\
    1,&t\ge\sigma.
  \end{cases}
\]
Thus $\widetilde\rho_t$ follows the actual ASB acquisition probability until
the first support failure and uses full acquisition only afterward. For this
proof-only process, continue Algorithm~\ref{alg:asb} with the feedback
generated under $\widetilde\rho_t$. Hence
\[
  \widetilde\rho_t(Z_t)
  \ge\min\{1,\alpha\ell_t(A_t)\}
  \qquad\text{for every }t,
\]
and Lemma~\ref{lem:policywise_regret_bound} gives
\[
  \Pr\!\left(
    \eqref{eq:fixed_rate_policy_regret_app}
    \text{ holds under }\widetilde\rho
  \right)
  \ge1-\delta/2.
\]

On $\mathcal E_{\mathrm{supp}}$, no support failure occurs, so
$\sigma=T+1$ and $\widetilde\rho_t=\rho_t$ for every $t$. Therefore the
proof-only modification is never used on the support event from Lemma~\ref{lem:hard_budget_acquisition_lower_bound}.
Combining the two failure probabilities gives
\[
  \Pr\!\left(
    \mathcal E_{\mathrm{supp}}
    \cap
    \left\{
      \eqref{eq:fixed_rate_policy_regret_app}
      \text{ holds for ASB}
    \right\}
  \right)
  \ge1-\delta.
\]

\emph{Exploration and learning-rate calibration.}
Under the theorem exploration schedule,
\[
  \underline\varepsilon
  =\min\{1/2,\sqrt{K L_T/T}\},
  \qquad
  \sum_{t=1}^T\varepsilon_t
  \le C\sqrt{KT L_T}.
\]
Set
\[
  G_{\mathrm{bd}}
  =K/(2\underline\varepsilon)+1/\alpha,
  \qquad
  \eta_\star
  =
  \min\left\{
    1/(2G_{\mathrm{bd}}),
    \sqrt{L_T/(3KT+T/\alpha)}
  \right\}.
\]

By Lemma~\ref{lem:policy_mixture_second_moment},
$KT+\mathsf D_T\le3KT+T/\alpha$. By the definition of
$\operatorname{App}_T(\Pi)$, substituting this bound and
$\eta=\eta_\star$ into \eqref{eq:fixed_rate_policy_regret_app} and taking
the maximum over $\pi\in\Pi$ gives, with probability at least $1-\delta$,
\[
  R_T-\operatorname{App}_T(\Pi)
  \le
  C\sqrt{KT L_T}
  +C\sqrt{L_T(3KT+T/\alpha)}
  +C G_{\mathrm{bd}}L_T.
\]

\emph{Rate simplification.}
Substituting $\alpha=B/\{KT\log(8T)\}$ and the theorem exploration rate yields
\[
  R_T-\operatorname{App}_T(\Pi)
  \le\widetilde O\!\left[
    \sqrt{KT L_T}
    +K L_T
    +T\sqrt{K L_T\log(8T)/B}
    +KT L_T\log(8T)/B
  \right].
\]
Taking the minimum with the deterministic bound $T$,
\[
  \min\{T,K L_T\}\le\sqrt{KT L_T},
  \qquad
  \min\{T,KT L_T\log(8T)/B\}
  \le
  T\sqrt{K L_T\log(8T)/B}.
\]
Since $L_T=\log(12NT/\delta)$, this gives
\[
  R_T-\operatorname{App}_T(\Pi)
  \le
  \widetilde O\!\left[
    \min\left\{
      T,\,
      \sqrt{KT\log(N/\delta)}
      \left(1+\sqrt{T/B}\right)
    \right\}
  \right],
\]
which is \eqref{eq:main_policy_regret}.

\subsection{Regret guarantee with full-acquisition initialization}
\label{app:bounded_initialization_extension}

Theorem~\ref{thm:main_asb_regret} analyzes the hard-budget rule without
full-acquisition initialization, while Algorithm~\ref{alg:hard_budget_pacing}
and the reported implementation allow one. Appendix~\ref{app:hard_budget_acquisition}
already accounts for these acquisitions in the same budget $B$; here we only
check that a short initialization preserves the regret rate.

\begin{proposition}[Regret with full-acquisition initialization]
Fix $c_{\mathrm{init}}>0$ and let
$\underline\varepsilon:=\min\{1/2,\sqrt{K L_T/T}\}$.
Suppose
\[
  C_{\mathrm{init}}\le B,
  \qquad
  C_{\mathrm{init}}\le c_{\mathrm{init}}K L_T.
\]
Use the theorem learning rate $\eta_\star$. Choose
$\varepsilon_{\mathrm{init}}\in[\underline\varepsilon,1/2]$, use
$\varepsilon_{\mathrm{init}}$ during initialization and
$\underline\varepsilon$ afterward, and use a regular interval partition of
the post-initialization rounds with $\zeta=\delta/(2J)$. Then, with probability at least
$1-\delta$,
\[
  R_T-\operatorname{App}_T(\Pi)
  \le
  \widetilde O\!\left[
    \min\left\{
      T,\,
      \sqrt{KT\log(N/\delta)}
      \left(1+\sqrt{T/B}\right)
    \right\}
  \right].
\]
\end{proposition}

\begin{proof}
\emph{Small-budget case.}
As in Section~\ref{app:main_regret_proof}, the deterministic bound
$R_T-\operatorname{App}_T(\Pi)\le T$ handles
$B<C_0K L_T$ for a sufficiently large fixed
$C_0=C_0(c_{\mathrm{init}},J,\kappa)$. Hence assume
$B\ge C_0K L_T$ and take $C_0\ge4c_{\mathrm{init}}$.

\emph{Support after initialization.}
Write $q:=B/T$. Then $C_{\mathrm{init}}\le B/4\le T/4$, and the
post-initialization remaining-budget rate is
\[
  q_{\mathrm{init}}
  :=(B-C_{\mathrm{init}})/(T-C_{\mathrm{init}})
  \ge3q/4.
\]
Initialization has $\rho_t=1$, so
\eqref{eq:relevance_support_requirement_app} holds there. For the post-initialization intervals, repeat the remaining-budget argument
in Lemma~\ref{lem:hard_budget_acquisition_lower_bound}, starting from
$q_{\mathrm{init}}$ instead of $q$. Whenever every earlier interval
$I_k$ satisfies $N_k^{\mathrm{floor}}\le B_k^{\mathrm{supp}}$, this gives
$q_j\ge q_{\mathrm{init}}$.

Since $C_{\mathrm{init}}\le T/4$, regularity gives
$L_j\ge3\kappa T/4-1$, which is at least $\kappa T/2$ once $C_0$ is
chosen sufficiently large. Repeating the Bernoulli KL bound from
Lemma~\ref{lem:hard_budget_acquisition_lower_bound} then gives
\[
  L_j\klbin\!\left(
    B_j^{\mathrm{tgt}}/L_j
    \middle\|
    q/\log(8T)
  \right)
  \ge \kappa B/32.
\]

Choose $C_0$ so that the assumption $B\ge C_0K L_T$ makes this at least
$\log(2J/\delta)$. The same conditional Chernoff step used in the proof of
Lemma~\ref{lem:hard_budget_acquisition_lower_bound} therefore shows that
$N_j^{\mathrm{floor}}\le B_j^{\mathrm{supp}}$ on each interval beginning
before deterministic final fill except with conditional probability at most
$\zeta$. A sequential union
bound then gives \eqref{eq:relevance_support_requirement_app} on all
post-initialization rounds with
probability at least $1-\delta/2$.

\emph{Additional exploration cost.}
Because $\underline\varepsilon\le\varepsilon_t\le1/2$, the bounds in
Section~\ref{app:gain_second_moment_control} remain unchanged. The only changed term in
Lemma~\ref{lem:policywise_regret_bound} is
\[
  \sum_{t=1}^T\varepsilon_t
  \le
  T\underline\varepsilon+C_{\mathrm{init}}/2
  \le
  C\sqrt{KT L_T}+(c_{\mathrm{init}}/2)K L_T.
\]
Since $K L_T\le B/C_0\le T$, the additional term is also
$O(\sqrt{KT L_T})$.

\emph{Applying the fixed-rate regret bound.}
Lemma~\ref{lem:policywise_regret_bound} assumes
\eqref{eq:relevance_support_requirement_app} on every round, whereas the
support argument in this proposition guarantees it with probability at least
$1-\delta/2$.

As in Section~\ref{app:main_regret_proof}, apply the lemma to
a proof-only process that follows ASB until a hypothetical first support
failure and sets the acquisition probability to one afterward. On this
support event no failure occurs, so the process is exactly ASB. Combining
the two failure probabilities and using the theorem learning rate
$\eta_\star$ proves the proposition.
\end{proof}

\subsection{Learned post-surrogate refinement in the regret bound}
\label{app:learned_refinement_efficiency}

Theorem~\ref{thm:main_asb_regret} does not require the fitted residual or
surrogate models to be correct. Here we ask when the learned post-surrogate
refinement nevertheless lowers the acquisition-dependent term $\mathsf D_T$
in the regret bound of Lemma~\ref{lem:policywise_regret_bound}.

Condition on $(\calF_{t-1},X_t,A_t)$. Let $P$ and $\widehat P$ be the true
and fitted conditional laws of $S$, and assume finite surrogate support. Let
$\widehat v_t(U_t,s)$ be the fitted squared-residual prediction underlying
$\widehat\tau_t$, with
$0\le\tau_t^2(U_t,s),\widehat v_t(U_t,s)\le V$. On a stochastic round with
$\overline c_{\mu,t}^2(A_t)>0$, we have
$\ell_t(A_t)>0$ and hence $\rho_t^{\mathrm{supp}}>0$. The fitted refinement solves
\[
  \rho_t^1(U_t,\cdot)
  \in
  \argmin_{\substack{
    \rho_t^{\mathrm{supp}}\le r(s)\le1\\
    \E_{\widehat P}r(S)=\rho_t^0
  }}
  \E_{\widehat P}\!\left[
    \widehat v_t(U_t,S)/r(S)
  \right].
\]

Because
$\rho_t^{\mathrm{supp}}\le\rho_t^0\le1$, the constant rule
$r(s)=\rho_t^0$ is feasible. Its fitted improvement under the displayed
objective is
\[
  \widehat\Delta_t(U_t)
  :=
  \E_{\widehat P}\!\left[
    \widehat v_t(U_t,S)/\rho_t^0
    -\widehat v_t(U_t,S)/\rho_t^1(U_t,S)
  \right]
  \ge0.
\]

Define
\[
  e_{{\mathrm{model}},t}(U_t)
  :=
  \E_P|\tau_t^2(U_t,S)-\widehat v_t(U_t,S)|
  +V\|P-\widehat P\|_1.
\]
The model-error term combines squared-residual prediction error with
surrogate-law error.

Let $\mathsf D_T^{\mathrm{pre}}$ be the quantity in
\eqref{eq:cumulative_acquisition_second_moment_app} obtained by replacing
$\rho_t(Z_t)$ with $\rho_t^0$ on stochastic rounds with
$\overline c_{\mu,t}^2(A_t)>0$, and leaving all other rounds unchanged.
Thus $\mathsf D_T^{\mathrm{pre}}$ removes only the post-surrogate
redistribution.

\begin{proposition}[Effect of learned post-surrogate refinement]
\label{prop:learned_refinement_gain}
If
$\rho_t=(1-\gamma_t)\rho_t^0+\gamma_t\rho_t^1$ on stochastic rounds,
then
\[
  \mathsf D_T^{\mathrm{pre}}-\mathsf D_T
  \ge
  \sum_{t=1}^T
  \E_t\!\left[
    \overline c_{\mu,t}^2(A_t)\gamma_t
    \left\{
      \widehat\Delta_t(U_t)
      -2e_{{\mathrm{model}},t}(U_t)/\rho_t^{\mathrm{supp}}
    \right\}
  \right],
\]
where the summand is zero on all other rounds. A positive right-hand side
therefore implies $\mathsf D_T<\mathsf D_T^{\mathrm{pre}}$.
\end{proposition}

\begin{proof}
Fix a stochastic round with $\overline c_{\mu,t}^2(A_t)>0$ and abbreviate
$m:=\rho_t^0$, $f:=\rho_t^{\mathrm{supp}}$,
$\widehat\Delta:=\widehat\Delta_t(U_t)$, and
$e_{\mathrm{model}}:=e_{{\mathrm{model}},t}(U_t)$.
For $v(S):=\tau_t^2(U_t,S)$ and every rule
$\varrho(S)\ge f$,
\[
  \left|
    \E_P\{v(S)/\varrho(S)\}
    -\E_{\widehat P}\{\widehat v_t(U_t,S)/\varrho(S)\}
  \right|
  \le e_{\mathrm{model}}/f.
\]

Apply this bound at $m$ and $\rho_t^1$. Convexity of $x\mapsto1/x$ for
$\rho_\gamma=(1-\gamma)m+\gamma\rho_t^1$ then gives
\[
  \E_P\{v(S)/m-v(S)/\rho_\gamma(S)\}
  \ge
  \gamma\{\widehat\Delta-2e_{\mathrm{model}}/f\}.
\]
Multiplying by $\overline c_{\mu,t}^2(A_t)$, averaging over
$A_t\sim p_t$, and summing over $t$ proves the result.
\end{proof}

Proposition~\ref{prop:learned_refinement_gain} does not require
\eqref{eq:relevance_support_requirement_app}. If that support condition also
holds, Lemma~\ref{lem:policywise_regret_bound} applies. A positive
right-hand side then makes the acquisition-dependent
$C\eta_\star\mathsf D_T$ term in
\eqref{eq:fixed_rate_policy_regret_app} smaller than the corresponding term
obtained from the pre-surrogate levels $\rho_t^0$.

\section{Proof of the Primary-Outcome Scarcity Lower Bound}

We prove Theorem~\ref{thm:audit_budget_lower_bound_main} within the setting
of Section~\ref{sec:problem}, using deterministic contexts, a constant
surrogate, and a policy class with zero approximation error.

\begin{proof}[Proof of Theorem~\ref{thm:audit_budget_lower_bound_main}]
\emph{Problem subclass.}
Set
\[
  M
  :=
  \min\left\{
    T,\,
    \left\lfloor\log N/\log K\right\rfloor
  \right\},
  \qquad
  \calA=[K].
\]
Initially take $\calX=[M]$ and fix
$X_1,\ldots,X_T\in[M]$, independently of the learner. Write
$T_x:=\sum_{t=1}^T\one\{X_t=x\}$ and choose the sequence so that
\[
  T_x\in\{\lfloor T/M\rfloor,\lceil T/M\rceil\}
  \qquad\text{for every }x\in[M].
\]
Set $S_t\equiv s_0$. For each
$\boldsymbol\vartheta=(\vartheta_1,\ldots,\vartheta_M)
\in\{0,1,\ldots,K\}^M$, let $\nu_{\boldsymbol\vartheta}$ denote the
environment with
\[
  r_{\boldsymbol\vartheta}^\star(x,a)
  :=
  1/2+\Delta\,\one\{\vartheta_x=a\},
  \qquad
  \Delta\in(0,1/4],
\]
and conditionally independent primary outcomes
$Y_t\sim\Bern\{r_{\boldsymbol\vartheta}^\star(X_t,A_t)\}$.

Let
\[
  \Pi_0:=\{\pi:[M]\to[K]\},
  \qquad
  |\Pi_0|=K^M\le N.
\]
If $K^M<N$, enlarge $\calX$ by enough context values that never appear in
$X_1,\ldots,X_T$, set all action means to $1/2$ on those unused contexts, and
extend $\Pi_0$ to a class $\Pi$ with $|\Pi|=N$ by varying policies only on
them; otherwise set $\Pi=\Pi_0$. For every $\boldsymbol\vartheta$, $\Pi$ contains a policy
that chooses action $\vartheta_x$ whenever $\vartheta_x\in[K]$; when
$\vartheta_x=0$, all actions have the same primary-outcome mean. Hence a
pointwise-optimal policy along $X_1,\ldots,X_T$ belongs to $\Pi$, so
$\operatorname{App}_T(\Pi)=0$.

We now fix an arbitrary learner satisfying
$\E_\nu\sum_{t=1}^TQ_t\le B$ for every environment $\nu$ in this
family and lower-bound its regret.

\emph{One-context comparison.}
For the lower-bound argument, let
$\boldsymbol\Theta=(\Theta_1,\ldots,\Theta_M)$ have independent
coordinates with
\[
  \Pr(\Theta_x=0)=1/2,
  \qquad
  \Pr(\Theta_x=i)=1/(2K),
  \quad i\in[K].
\]
Fix $x$ and hold the other environment coordinates fixed by conditioning on
$\Theta_{-x}=\vartheta_{-x}$. Under this conditioning, write $P_0$ and
$P_i$ for the laws of
\[
  \bigl(X_t,A_t,S_t,Q_t,Q_tY_t\bigr)_{t=1}^T
\]
observed by the learner when $\Theta_x=0$ and $\Theta_x=i$, respectively,
and let $\E_0,\E_i$ denote the corresponding expectations. Define
\[
  n_{x,i}^{\rm play}
  :=\sum_{t=1}^T\one\{X_t=x,A_t=i\},
  \qquad
  n_{x,i}^{\rm acq}
  :=\sum_{t=1}^TQ_t\one\{X_t=x,A_t=i\}.
\]
The fixed contexts $X_t$ and constant surrogates $S_t$ are identical under
$P_0$ and $P_i$. Given the same past observations and current $X_t$, the
learner uses the same conditional law for $A_t$; given also $A_t$ and $S_t$,
it uses the same conditional law for $Q_t$. The only differing conditional
law is therefore that of $Y_t$ when $X_t=x$, $A_t=i$, and $Q_t=1$.
The adaptive KL chain rule gives
\[
  \KL(P_0\|P_i)
  =
  \klbin(1/2\|1/2+\Delta)\,
  \E_0 n_{x,i}^{\rm acq}
  \le
  4\Delta^2\E_0 n_{x,i}^{\rm acq}.
\]
Pinsker's inequality converts this into a bound on the play count:
\[
  \E_i n_{x,i}^{\rm play}
  -\E_0 n_{x,i}^{\rm play}
  \le
  T_x\sqrt{\KL(P_0\|P_i)/2}
  \le
  \sqrt2\,T_x\Delta
  \sqrt{\E_0 n_{x,i}^{\rm acq}}.
\]

Under $P_0$, the $K$ play counts at context $x$ sum to $T_x$. Write
\[
  \bar n_x^{\rm acq}(\vartheta_{-x})
  :=
  \E_0\sum_{i=1}^K n_{x,i}^{\rm acq}
\]
for the expected number of primary acquisitions made at context $x$ under
$P_0$. Averaging the preceding inequality over $i\in[K]$ and
applying Cauchy--Schwarz to the square-root acquisition counts gives
\[
\begin{aligned}
  \frac1K\sum_{i=1}^K\E_i n_{x,i}^{\rm play}
  &\le
  T_x/K
  +\frac{\sqrt2\,T_x\Delta}{K}
    \sum_{i=1}^K\sqrt{\E_0 n_{x,i}^{\rm acq}}\\
  &\le
  T_x/K
  +\sqrt2\,T_x\Delta
   \sqrt{\bar n_x^{\rm acq}(\vartheta_{-x})/K}.
\end{aligned}
\]

If $\Theta_x=0$, rounds with $X_t=x$ incur no regret. If
$\Theta_x=i$, each such round on which action $i$ is not played incurs
regret $\Delta$. Conditional on
$\Theta_{-x}=\vartheta_{-x}$, averaging over $\Theta_x$ therefore gives
\[
\begin{aligned}
  \frac{\Delta}{2K}
  \sum_{i=1}^K
  \left\{
    T_x-\E_i n_{x,i}^{\rm play}
  \right\}
  &=
  \frac{\Delta}{2}
  \left\{
    T_x-\frac1K\sum_{i=1}^K\E_i n_{x,i}^{\rm play}
  \right\}\\
  &\ge
  \frac{\Delta}{2}
  \left[
    T_x(1-1/K)
    -\sqrt2\,T_x\Delta
     \sqrt{\bar n_x^{\rm acq}(\vartheta_{-x})/K}
  \right].
\end{aligned}
\]

\emph{Aggregate over contexts.}
Average the preceding bound over $\Theta_{-x}$ and use Jensen's inequality
for the square root. Write
\[
  \bar n_x^{\rm acq}
  :=
  \E_{\Theta_{-x}}
  \bar n_x^{\rm acq}(\Theta_{-x}).
\]
The expected-budget condition also holds after averaging over
$\boldsymbol\Theta$. For each $x$, retaining only the event
$\Theta_x=0$, which has probability $1/2$, gives
\[
  B
  \ge
  \E_{\boldsymbol\Theta}
  \E_{\nu_{\boldsymbol\Theta}}
  \sum_{t=1}^TQ_t
  =
  \sum_{x=1}^M
  \E_{\boldsymbol\Theta}
  \E_{\nu_{\boldsymbol\Theta}}
  \sum_{i=1}^K n_{x,i}^{\rm acq}
  \ge
  (1/2)\sum_{x=1}^M\bar n_x^{\rm acq}.
\]
Hence $\sum_x\bar n_x^{\rm acq}\le2B$. The choice of
$X_1,\ldots,X_T$ also gives
\[
  \sum_{x=1}^M T_x^2
  \le
  (\max_xT_x)\sum_xT_x
  \le
  2T^2/M.
\]
Summing the one-context regret bound, using $1-1/K\ge1/2$, and applying
Cauchy--Schwarz to
$\sum_xT_x\sqrt{\bar n_x^{\rm acq}}$ now yields
\[
  \E_{\boldsymbol\Theta}
  \E_{\nu_{\boldsymbol\Theta}}R_T
  \ge
  \Delta T/4
  -\sqrt2\,\Delta^2T\sqrt{B/(KM)}.
\]

\emph{Choose the gap.}
Set
\[
  \Delta
  :=
  (1/16)
  \min\left\{
    1,\sqrt{KM/B}
  \right\}.
\]
Then
\[
  \E_{\boldsymbol\Theta}
  \E_{\nu_{\boldsymbol\Theta}}R_T
  \ge
  (T/128)
  \min\left\{
    1,\sqrt{KM/B}
  \right\}.
\]
Hence at least one environment $\nu_{\boldsymbol\vartheta}$ satisfies
the same lower bound.

If $\log N/\log K\le T$, then
$M=\lfloor\log N/\log K\rfloor
\ge(1/2)\log N/\log K$, using $N\ge K$. If
$\log N/\log K>T$, then $M=T$, and $B\le T$ and $K\ge2$ make the
preceding lower bound at least $T/128$. Therefore, for a universal constant
$c_{\mathrm{lb}}>0$,
\[
  \sup_\nu\E_\nu R_T
  \ge
  c_{\mathrm{lb}}
  \min\left\{
    T,\,
    \sqrt{\frac{KT\log N}{\log K}}
    \sqrt{\frac{T}{B}}
  \right\}.
\]
The learner was arbitrary among those satisfying
$\sup_\nu\E_\nu\sum_{t=1}^TQ_t\le B$, so this proves the theorem.
\end{proof}

\section{One Fitting Specification for the Learned Components}
\label{app:practical_fitted_components}

This appendix gives one compact implementation example for the three fitted
objects used by finite-policy \ASB: $h_t$, $\widehat\tau_t$, and
$\widehat\E_t(\cdot\mid U_t)$. These are implementation details rather than
additional assumptions of \ASB\ or Theorem~\ref{thm:main_asb_regret}.
Appendix~\ref{app:experimental_details} records the experiment-specific
choices and replacements.

\paragraph{Bridge fit.}
Let $\phi_{\mathrm{ctx}}(x)\in\mathbb R^{d_{\mathrm{ctx}}}$ be an
observable context feature map, let
$\phi(x)=(1,\phi_{\mathrm{ctx}}(x)^\top)^\top$, and let $\psi(s)$ be a
bounded surrogate feature. We use
\[
  h_t(x,a,s)
  =\clip_{[0,1]}\left[
    \phi(x)^\top\widehat\theta^Y_{t,a}
    +\widehat\beta_t
      \{\psi(s)-\phi(x)^\top\widehat\theta^S_{t,a}\}
  \right].
\]
The coefficients $\widehat\theta^S_{t,a}$ are fitted from past surrogate
observations, while $\widehat\theta^Y_{t,a}$ and $\widehat\beta_t$ are fitted
from previously acquired primary outcomes.

\paragraph{Residual RMS fit.}
Acquired primary outcomes provide squared bridge-residual labels. We fit these
labels by ridge regression using the residual-regression feature vector
\[
\begin{aligned}
  \chi(x,a,s)
  &=\big[1,\ \phi_{\mathrm{ctx}}(x),\ e_a,\ \zeta(s),\
         \zeta(s)\phi_{\mathrm{ctx}}(x),\ \zeta(s)e_a\big],\\
  \widehat v_t(x,a,s)
  &=\chi(x,a,s)^\top\widehat\theta^v_t.
\end{aligned}
\]
where $e_a$ is the action-indicator vector, $\zeta(s)$ is a bounded surrogate
marker, and $\widehat\theta^v_t$ is the ridge coefficient vector. Thus
$\widehat v_t$ is the fitted squared-residual prediction. For a fixed RMS cap
and scale $\tau_{\max}>0$, we clip $\widehat v_t$ below at a fixed positive
implementation floor and above at $\tau_{\max}^2$, then take its square root
and divide by $\tau_{\max}$ to obtain $\widehat\tau_t(x,a,s)\in(0,1]$.

\paragraph{Conditional surrogate fit.}
Let $\widehat P_t(\cdot\mid U_t)$ denote the fitted conditional surrogate
law. When $U_t$ and $S$ take values in finite supports, we estimate it by
smoothed empirical conditional frequencies from past surrogate observations
and let $\widehat\E_t(\cdot\mid U_t)$ denote expectation under
$\widehat P_t$.
This fit uses no primary outcomes. Appendix~\ref{app:experimental_details}
reports experiment-specific alternatives.

\section{A Linear Analogue of the Acquisition Roles}
\label{app:linear_analogue}

This appendix checks that the acquisition roles in
Section~\ref{sec:policy_acquisition} are not specific to finite policy
enumeration. Under a linear primary-outcome mean model, leverage provides a
pre-surrogate relevance factor, while residual RMS retains its
post-surrogate allocation role. We give one budgeted linear construction and
a regret bound. The construction is separate from Algorithm~\ref{alg:asb}.

\subsection{One linear construction}

Assume $\varphi(x,a),\theta^\star\in\mathbb R^d$ and
\[
  r^\star(x,a)=\varphi(x,a)^\top\theta^\star,
  \qquad
  \|\varphi(x,a)\|_2\le1,
  \qquad
  \|\theta^\star\|_2\le1,
  \qquad
  r^\star(x,a)\in[0,1].
\]
Let $h_t:\calZ\to[0,1]$ and $\widehat\tau_t:\calZ\to[0,1]$ be fitted
using data from earlier rounds. After observing $S_t$, the learner evaluates
them at the current state $Z_t$. The current $Y_t$ is not used before the
acquisition decision.

\paragraph{Action selection and pre-surrogate leverage.}
Initialize $V_1=\lambda_{\rm lin}I$ with $\lambda_{\rm lin}\ge4$ and
$\widehat\theta_1=0$. At the beginning of round $t$, $V_t$ and
$\widehat\theta_t$ summarize the weighted updates from earlier rounds.
Given $X_t$, define
\[
  \upsilon_t(a)
  :=\|\varphi(X_t,a)\|_{V_t^{-1}}
\]
and, for a confidence radius $\beta_{\rm lin}>0$, choose
\[
  A_t\in\argmax_a
  \left\{
    \varphi(X_t,a)^\top\widehat\theta_t
    +\beta_{\rm lin}\upsilon_t(a)
  \right\}.
\]
Write $\varphi_t:=\varphi(X_t,A_t)$ and
$\upsilon_t:=\upsilon_t(A_t)$. Before observing $S_t$, set
\[
  \rho_t^{\rm lev}:=\min\{1,c\upsilon_t\}
\]
for a scale $c>0$. Thus $A_t$ and $\rho_t^{\rm lev}$ are both fixed
before $S_t$ is observed.

\paragraph{Post-surrogate acquisition and update.}
After observing $S_t$, add a bounded residual component:
\[
  \rho_t^{\rm res}(Z_t)
  :=\min\{1,c\upsilon_t\widehat\tau_t(Z_t)\},
  \qquad
  \widetilde\rho_t(Z_t)
  :=\min\{1,\rho_t^{\rm lev}+\rho_t^{\rm res}(Z_t)\}.
\]
The leverage term sets the pre-surrogate acquisition level. After observing
$S_t$, the residual term adds probability where the learned residual RMS is
larger; the regret result below only requires this added term to be bounded.

While fewer than $B$ primary outcomes have been acquired, proceed as
follows. If $\upsilon_t=0$, set $Q_t=0$ and make no update:
\[
  V_{t+1}=V_t,
  \qquad
  \widehat\theta_{t+1}=\widehat\theta_t.
\]
If $\upsilon_t>0$, draw
$Q_t\sim\Bern\{\widetilde\rho_t(Z_t)\}$ and form
\[
  \Phi_t
  =h_t(Z_t)
   +\frac{Q_t}{\widetilde\rho_t(Z_t)}
      \{Y_t-h_t(Z_t)\}.
\]
Update
\[
  V_{t+1}=V_t+\rho_t^{\rm lev}\varphi_t\varphi_t^\top,
  \qquad
  \widehat\theta_{t+1}
  =V_{t+1}^{-1}\{V_t\widehat\theta_t+\rho_t^{\rm lev}\varphi_t\Phi_t\}.
\]
Once $B$ primary outcomes have been acquired, the learner stops acquiring
primary outcomes and no longer updates the linear learner.

On any round that updates the linear learner, Equation~\eqref{eq:target_validity}
gives
\[
  \E(\Phi_t\mid\calF_{t-1},U_t)
  =
  r^\star(X_t,A_t)
  =
  \varphi_t^\top\theta^\star.
\]
Because $\rho_t^{\rm lev}\varphi_t$ is fixed before $S_t$,
\[
  \E\!\left(
    \rho_t^{\rm lev}\varphi_t\Phi_t
    \mid\calF_{t-1},U_t
  \right)
  =
  \rho_t^{\rm lev}\varphi_t\varphi_t^\top\theta^\star.
\]
The inverse-probability correction on such a round is bounded by
\[
  \left|
    \rho_t^{\rm lev}
    \frac{Q_t}{\widetilde\rho_t(Z_t)}
    \{Y_t-h_t(Z_t)\}
  \right|
  \le
  \frac{\rho_t^{\rm lev}}{\widetilde\rho_t(Z_t)}
  \le1,
\]
which is used in the confidence argument below.

\subsection{Regret bound for the linear construction}

For the construction above, set
\[
  b:=B/8,
  \qquad
  \mathfrak L_T^{\rm lin}
  :=2d\log\!\left(1+T/(\lambda_{\rm lin}d)\right),
  \qquad
  c:=b^{3/2}/\{2T\sqrt{\mathfrak L_T^{\rm lin}}\}.
\]
For sufficiently large universal constants $c_{\rm conf}$ and $C$, define
\[
  \mathfrak C_T^{\rm lin}
  :=
  c_{\rm conf}
  \left\{
    d\log(2+T/\lambda_{\rm lin})
    +\log(2c_{\rm conf}T/\delta)
  \right\},
\]
and use the action rule above with
\[
  \beta_{\rm lin}
  :=
  \sqrt{\lambda_{\rm lin}}
  +C\left\{
    \sqrt{
      2\log(1+T/\lambda_{\rm lin})\,
      \mathfrak C_T^{\rm lin}
    }
    +\mathfrak C_T^{\rm lin}/\sqrt{\lambda_{\rm lin}}
  \right\}.
\]

\begin{proposition}[Regret for the linear construction]
For an appropriate universal choice of constants and
$\lambda_{\rm lin}\asymp
d\log(2T/\delta)+\log(2T/\delta)$, the learner satisfies the hard
primary-outcome budget in Section~\ref{sec:problem} and, with probability at
least $1-\delta$,
\[
  R_T
  =
  \widetilde O\!\left[
    \min\left\{
      T,\,
      d\sqrt T\,\sqrt{T/B}
    \right\}
  \right].
\]
\end{proposition}

\begin{proof}
For the proof, ignore the budget stop and continue the displayed acquisition
and linear updates through $T$.

\emph{Leverage and total acquisition probability.}
The determinant lemma gives
\[
  \sum_{t=1}^T
  \rho_t^{\rm lev}\upsilon_t^2
  \le
  2\log\{\det(V_{T+1})/\det(V_1)\}
  \le
  \mathfrak L_T^{\rm lin}.
\]
Split the rounds according to $c\upsilon_t<1$ or
$c\upsilon_t\ge1$. H\"older's inequality and the preceding bound give
\[
  \sum_{t=1}^T \rho_t^{\rm lev}\le b,
  \qquad
  \sum_{t=1}^T \upsilon_t
  \le
  2T\sqrt{\mathfrak L_T^{\rm lin}/b}.
\]
Since $\rho_t^{\rm res}(Z_t)\le c\upsilon_t$,
\[
  \sum_{t=1}^T\widetilde\rho_t(Z_t)
  \le
  \sum_t\rho_t^{\rm lev}
  +c\sum_t\upsilon_t
  \le
  2b
  =
  B/4.
\]

\emph{Uniform confidence.}
Fix a Euclidean unit vector $e$ and write
\[
  a_t(e)
  :=
  \frac{\sqrt{\rho_t^{\rm lev}}\,e^\top\varphi_t}
       {\sqrt{e^\top V_te}}.
\]
Cauchy--Schwarz gives the first inequality below, while
$e^\top V_{t+1}e/e^\top V_te=1+a_t^2(e)$ gives the second by telescoping:
\[
  |a_t(e)|
  \le
  \sqrt{\rho_t^{\rm lev}}\upsilon_t,
  \qquad
  \sum_{t=1}^T a_t^2(e)
  \le
  2\log(1+T/\lambda_{\rm lin}).
\]

Set $M_{t,e}=0$ when $\upsilon_t=0$. Otherwise define
\[
  M_{t,e}
  :=
  \frac{\rho_t^{\rm lev}e^\top\varphi_t}
       {\sqrt{e^\top V_te}}
  \{\Phi_t-r^\star(X_t,A_t)\}.
\]
Target validity gives
\[
  \E(M_{t,e}\mid\calF_{t-1},U_t)=0.
\]
Conditionally on $(\calF_{t-1},Z_t)$,
\[
  \E\!\left[
    \{\Phi_t-r^\star(X_t,A_t)\}^2
    \mid\calF_{t-1},Z_t
  \right]
  \le
  1+\frac{1}{\widetilde\rho_t(Z_t)}.
\]
Using $\rho_t^{\rm lev}\le\widetilde\rho_t(Z_t)$,
$h_t(Z_t),Y_t,r^\star(X_t,A_t)\in[0,1]$, and Cauchy--Schwarz, we obtain
\[
  \E(M_{t,e}^2\mid\calF_{t-1},U_t)\le2a_t^2(e),
  \qquad
  |M_{t,e}|\le2\upsilon_t\le2/\sqrt{\lambda_{\rm lin}}.
\]
Maximal Freedman therefore gives, with failure probability at most
$2e^{-u}$,
\[
  \sup_{n\le T+1}
  \left|
    \sum_{t<n} M_{t,e}
  \right|
  \le
  C\left\{
    \sqrt{2\log(1+T/\lambda_{\rm lin})\,u}
    +u/\sqrt{\lambda_{\rm lin}}
  \right\}.
\]
Discrete summation by parts bounds
\[
  \frac{
    \left|
      e^\top\sum_{t<n}
      \rho_t^{\rm lev}\varphi_t
      \{\Phi_t-r^\star(X_t,A_t)\}
    \right|
  }{
    \sqrt{e^\top V_ne}
  }
\]
by twice the preceding right-hand side.
Since
$\lambda_{\rm lin}I\preceq V_t\preceq(\lambda_{\rm lin}+T)I$,
a Euclidean $\epsilon$-net with
$\epsilon\asymp
\sqrt{\lambda_{\rm lin}/(\lambda_{\rm lin}+T)}$ gives, by a union
bound,
\[
  \left\|
    \sum_{i<t}
    \rho_i^{\rm lev}\varphi_i
    \{\Phi_i-r^\star(X_i,A_i)\}
  \right\|_{V_t^{-1}}
  \le
  \beta_{\rm lin}-\sqrt{\lambda_{\rm lin}}
\]
simultaneously for all $t\le T+1$ with probability at least
$1-\delta/2$. Since
\[
  V_t(\widehat\theta_t-\theta^\star)
  =
  \sum_{i<t}
  \rho_i^{\rm lev}\varphi_i
  \{\Phi_i-r^\star(X_i,A_i)\}
  -\lambda_{\rm lin}\theta^\star
\]
and
$\|\lambda_{\rm lin}\theta^\star\|_{V_t^{-1}}
\le\sqrt{\lambda_{\rm lin}}$, we obtain
\[
  \|\widehat\theta_t-\theta^\star\|_{V_t}
  \le
  \beta_{\rm lin}
  \qquad
  \text{for all }t\le T+1.
\]

\emph{Budget stop and regret.}
On this confidence event, optimism gives
\[
  R_T
  \le
  2\beta_{\rm lin}\sum_{t=1}^T\upsilon_t
  \le
  4T\sqrt{\mathfrak L_T^{\rm lin}/b}\,\beta_{\rm lin}
  =
  \widetilde O\!\left(
    d\sqrt T\,\sqrt{T/B}
  \right)
\]
before reimposing the budget stop.

Let $\widetilde N_T$ be the corresponding number of acquisitions. Since
$\sum_t\widetilde\rho_t(Z_t)\le B/4$ pathwise,
\[
  \E e^{\lambda\widetilde N_T}
  \le
  \exp\{(e^\lambda-1)B/4\}
  \qquad (\lambda>0).
\]
Taking $\lambda=\log4$ gives
\[
  \Pr(\widetilde N_T\ge B)
  \le
  \exp\{-(\log4-3/4)B\}.
\]
If
$B\ge(\log4-3/4)^{-1}\log(2/\delta)$, this probability is at most
$\delta/2$. On the complementary event, $\widetilde N_T<B$, so the budget stop is never
invoked and the displayed regret bound applies to the actual learner.
Combining the two events proves the proposition. For smaller $B$, the deterministic bound
$R_T\le T$ is absorbed by the displayed $\widetilde O$ rate.
\end{proof}

\section{Experimental Details and Reproducibility}
\label{app:experimental_details}

This appendix records the experiment-specific settings used in
Section~\ref{sec:experiments}. Appendix~\ref{app:practical_fitted_components}
gives one fitted-object specification; here we state the settings actually
used in the reported experiments. Exact refit schedules and low-level runtime
bookkeeping are given in the supplementary code.

\subsection{Common finite-policy implementation}
\label{app:common_finite_policy_implementation}

All budget-matched methods use the augmented feedback in
\eqref{eq:audited_reward} and the same second-order policy update in
\eqref{eq:observed_second_order_update}. Adaptive acquisition methods also
use the exact-count budget manager in
Appendix~\ref{app:hard_budget_acquisition}.

\paragraph{Refinement strength.}
All reported fitted \ASB\ runs use $\gamma_{\mathrm{prac}}=0.75$ on
stochastic refinement rounds. It is fixed before $S_t$ and is a design
constant rather than a fitted quantity. On deterministic initialization and
final-fill rounds, $\rho_t^0=\rho_t^1$.

\paragraph{Policy learner.}
All budget-matched learned methods use the same adaptive finite-sample
learning-rate schedule. When that schedule changes, the current gain is
applied first and the next-round policy mixture receives a $1/T$ uniform
component. Theorem~\ref{thm:main_asb_regret} instead uses the fixed learning
rate defined in Appendix~\ref{app:main_regret_proof}.

\paragraph{Initialization and budget accounting.}
The implementation begins with a fully acquired prefix. Let
$C_{\mathrm{init}}$ be the first prespecified checkpoint no smaller than
$\lceil4K\log(8NT/0.05)\rceil$; its checkpoint sequence is given in the
supplementary code. The fitted methods acquire every primary outcome for
rounds $1,\ldots,C_{\mathrm{init}}$, and these acquisitions count toward
$B$. The remaining rounds are split into three consecutive budget intervals
with relative lengths $1/4,3/8,3/8$, after integer rounding, and use
$\zeta=0.01$ in the support reserve.

All reported hard-budget runs finish with exactly $B$ acquisitions. The
current primary outcome is still unobserved when the action and acquisition
probability are chosen. Only the exact PRE and POST references use oracle
residual-RMS quantities.

\subsection{Controlled synthetic family}
\label{app:controlled_synthetic_family}

\paragraph{Primary problem.}
The controlled family has $512$ contexts, $K=8$ actions, and $N=256$
represented policies. Policy-disagreement contexts have total probability
$1/3$ and policy-agreement contexts probability $2/3$, with contexts uniform
within each region. Each context has one optimal action $a^\star(x)$, with
\[
  r^\star\{x,a^\star(x)\}=1/2+\Delta_x/2,\qquad
  r^\star(x,a)=1/2-\Delta_x/2\quad(a\ne a^\star(x)),
\]
where $\Delta_x=0.12$ in the disagreement region and $0.24$ in the agreement
region. The observed context feature is a $16$-dimensional one-hot encoding
of the region and the optimal action. The policy class contains the
pointwise-optimal policy and $255$ competitors that differ from it only on
disagreement-region contexts, with prespecified disagreement probabilities spanning
$0.10$ to $0.90$; the same realized class is used in every variant.

Let $R\in\{-1,+1\}$ take each sign with probability $1/2$, independently of
$H\sim\Bern(q)$ with $q=0.10$. With $\sigma=0.07$, $m_0=0.20$, and
\[
  m_1=\left\{\frac{1-(1-q)m_0^2}{q}\right\}^{1/2},
  \qquad
  Y=r^\star(X,A)+\sigma m_HR.
\]
The surrogate has four levels with marginal probabilities
$(0.30,0.30,0.30,0.10)$. Let $M$ indicate the fourth level; when $M=0$, one
of the first three levels is chosen uniformly.

\paragraph{Information-order experiment.}
The experiment changes only the association between $M$ and the
residual-magnitude state $H$:
\[
  \Pr(M=1\mid H=1)=q+\lambda(1-q),\qquad
  \Pr(M=1\mid H=0)=q(1-\lambda),\qquad \lambda\in[0,1].
\]
At $\lambda=0$, $M$ is independent of residual magnitude; at $\lambda=1$,
it identifies the high-magnitude state. Since only this association changes
and $R$ is independent and symmetric, $P(Y\mid X,A)$ is fixed across
$\lambda$ and $\E(Y\mid X,A,S)=\E(Y\mid X,A)$. The middle panel of
Figure~\ref{fig:controlled_mechanisms} uses
$\lambda\in\{0,.10,.20,.25,.30,.40,.50,.60,.70,.75,.80,.90,1\}$; the
component comparison uses $\lambda=1$.

\paragraph{Information-placement experiment.}
Let $v_x=\lambda_x^2$ denote the normalized Brier-score improvement for
predicting $H$ from the surrogate. For
average information $v\in\{.25,.50,.75\}$ and alignment
$a\in\{-1,-.75,-.50,-.25,0,.25,.50,.75,1\}$, set
\[
  d=\frac{a}{2}\min\{v,1-v\},\qquad
  v_{\rm dis}=v+2d,\qquad
  v_{\rm agr}=v-d,
  \qquad
  \lambda_{\rm dis}=\sqrt{v_{\rm dis}},\quad
  \lambda_{\rm agr}=\sqrt{v_{\rm agr}}.
\]
Then $(1/3)v_{\rm dis}+(2/3)v_{\rm agr}=v$. The values $a=-1,0,1$
correspond to Away, Neutral, and Toward in Figure~\ref{fig:controlled_mechanisms}.
The primary law, context distribution, and policy class do not change.

The synthetic experiments use the three fitted objects as specified in
Appendix~\ref{app:practical_fitted_components}; fits that require primary
outcomes use only acquired primary outcomes.

\paragraph{Comparators.}
\emph{Relevance only} uses the same decision-relevance allocation as \ASB\
but no $S_t$-dependent residual refinement. \emph{Residual only} retains the
learned current-surrogate refinement but removes decision relevance.
\emph{Uniform} uses neither score, and \emph{Fixed-N} spends the budget on
the first $B$ rounds. The exact PRE and POST references use the exact residual
allocation before and after observing $S_t$, respectively.

\paragraph{Randomization.}
The procedures are paired within the same 64 environments. Plotted summaries
show mean $\pm1$ SE; contrasts computed within an environment use the SE of
the paired world-level differences.

\subsection{KuaiRec data-derived benchmark}
\label{app:kuairec_benchmark}

\paragraph{Benchmark construction.}
Using KuaiRec's nearly fully observed small matrix \citep{gao2022kuairec}, we
partition eligible videos deterministically into five-video candidate sets.
Each round is a user--candidate-set pair for which all five user--video entries
are observed. After selecting one video, the learner observes
$S7=\one\{\text{play duration}>7\text{ seconds}\}$ and then decides whether
to acquire $Y=\min\{\text{watch ratio},1\}$. The benchmark has
$T=22{,}417$, $K=5$, and nine represented policies: five fixed
candidate-position policies and four learned recommenders (one
side-information factorization, one pairwise ranker, and two two-tower
models). Exact policy construction and candidate-set eligibility are given in
the supplementary code.

\paragraph{Fitted acquisition models.}
The bridge and residual second-moment models are gradient-boosted regressors;
out-of-fold bridge residuals provide the squared-residual target. The
conditional surrogate law is a gradient-boosted classifier. All three models
are fitted on a disjoint KuaiRec subset and then held fixed during benchmark
evaluation. Benchmark primary outcomes are not used to refit them.

As a retrospective check of the residual model,
$\operatorname{Spearman}(\widehat\tau,|Y-h|)=0.4387$ across selectable
user--video cells, with a positive association within each $S7$ stratum. This
check is computed only after the online run.

\paragraph{Evaluation.}
We use primary-outcome budget fractions
$2,3,4,6,8,12,16,$ and $24\%$. At each budget, ASB and Relevance only use
the same 32 order/action randomizations; each is rerun with 8 acquisition
randomizations for every order/action stream. After each online run, the
complete benchmark table provides the exact primary value of every represented
policy for scoring the final policy mixture. Unacquired benchmark outcomes
and these full-table policy values are never provided to the online learner.

The plotted SEs use both the 32 order/action streams and the 8 acquisition
randomizations within each stream. They summarize algorithmic randomization
on this fixed benchmark table, not uncertainty for a future user population.
At the $2\%$ budget, the 346-round fully acquired prefix uses $77.2\%$ of
the available primary-outcome acquisitions.

\section{Extended Related Work}
\label{app:related_scope}

This appendix organizes related work by information timing and the role of
costly measurement. \ASB\ studies post-action surrogate observation followed
by same-round acquisition of the primary outcome that drives future policy
learning and regret.

\paragraph{Prediction-guided measurement and selective ground-truth acquisition.}
Prediction-guided experimentation, active inference, and prediction-powered
sampling use predictions or surrogates to allocate costly ground-truth
measurements for estimation, testing, prediction, or identification
\citep{ao2024pgae,ao2026llmaudits,park2026surrogate_sampling,zrnic2024active,
ashouritaklimi2026ppat,brawand2026amppi,mozer2026stratified}. Missing-data and
doubly robust methods provide related correction tools
\citep{robins1994estimation,dudik2011doubly}, while prediction-powered
surrogate methods study efficiency when inexpensive predictions are observed
more broadly than the true outcome \citep{ji2025predictions}. \ASB\ instead
places same-round primary-outcome acquisition inside an evolving contextual
policy learner: after action selection, the current surrogate informs
acquisition through residual RMS and current policy-comparison relevance.
The acquired outcomes update future policies under one horizon-wide hard
budget, and Section~\ref{sec:theory_efficiency} isolates PRE versus POST
acquisition under cumulative regret.

\paragraph{Costly feedback and observation order.}
Budgeted-feedback bandits decide when to query reward
\citep{efroni2021confidence,tucker2023costly}; preference-query contextual
bandits decide when to request comparison feedback \citep{sekhari2023active};
and budgeted information-reveal models control when current information is
delivered to the action learner \citep{gan2024reveal}. Other bandit models use
information revealed after acting, including intermediate observations,
post-serving contexts, and post-action contexts
\citep{esposito2023intermediate,wang2023followups,shahverdikondori2025postaction}.
In \ASB\ the timing is $X_t\to A_t\to S_t\to Q_t\to Q_tY_t$: the surrogate
cannot change the current action but can determine whether its primary outcome
is acquired for future learning under the hard budget.

\paragraph{Surrogate-guided sequential decisions.}
Progressive-feedback and surrogate-guided adaptive designs use early outcomes
to accelerate learning about delayed objectives or adapt future treatment
allocation
\citep{mcdonald2023impatient,zhang2025impatient,shi2024surrogates,
zhang2026metalevel}. Surrogate-reward bandits may incorporate proxy reward
signals directly into reward prediction or policy learning
\citep{sridhar2026correlation}, while individualized-treatment-rule work
evaluates decisions learned from surrogate endpoints
\citep{xu2025itr_surrogates}. In \ASB\, the surrogate is acquisition-side
information rather than a substitute for the primary outcome: action value and
regret remain defined by the primary outcome.

\paragraph{Clinical surrogates and adjacent acquisition problems.}
Clinical surrogate and limited-follow-up work studies when earlier or cheaper
outcomes can support inference on a primary endpoint or reduce primary-outcome
measurement burden
\citep{prentice1989surrogate,kosorok1993surrogate,kallus2025surrogates}.
Active feature acquisition instead purchases input variables to improve a
current prediction or decision \citep{li2021activefeature,valancius2024acquisition}.
In \ASB\, the surrogate has already been observed when acquisition is
considered, and the purchased variable is the primary outcome itself.

\paragraph{Contextual-bandit learner backbone.}
Finite-policy contextual aggregation provides the classical
$\sqrt{KT\log N}$ policy-class regret scale
\citep{beygelzimer2011contextual}, while regression-oracle and
function-approximation methods provide other contextual-bandit learner
architectures \citep{foster2020beyond,ye2023uncertainty,jia2024variance}.
These works concern the action-learning backbone rather than the
post-surrogate primary-outcome acquisition decision studied here.
Appendix~\ref{app:linear_analogue} gives a linear analogue in which leverage
supplies an analogous pre-surrogate relevance factor.

\end{document}